\documentclass[journal,twoside]{IEEEtran}

\usepackage{cite}
\usepackage{amsmath,amssymb}
\usepackage{graphicx}
\usepackage{url}
\usepackage{amsmath,amssymb,booktabs,makecell,array,threeparttable,cite,url}
\usepackage[T1]{fontenc}
\usepackage{listings}

\providecommand{\Y}{\ensuremath{\checkmark}}
\providecommand{\N}{\ensuremath{\times}}
\providecommand{\Pt}{\ensuremath{\circ}}
\providecommand{\NS}{--}

\usepackage{array}
\usepackage{makecell}
\usepackage{tabularx}

\newcolumntype{Y}{>{\centering\arraybackslash}X}
\newcolumntype{L}{>{\raggedright\arraybackslash}X}
\usepackage{orcidlink}
\hypersetup{hidelinks}
\usepackage{listings}
\usepackage{booktabs}
\usepackage{tabularx}
\usepackage{mathrsfs}
\def\BibTeX{{\rm B\kern-.05em{\sc i\kern-.025em b}\kern-.08em
T\kern-.1667em\lower.1ex\hbox{E}\kern-.125emX}}

\usepackage{amsmath,amssymb}
\usepackage{graphicx,booktabs,cite,url}

\newcommand{\R}{\mathbb{R}}
\newcommand{\SE}{\mathrm{SE}(3)}
\newcommand{\Ad}{\operatorname{Ad}}
\newcommand{\Log}{\operatorname{Log}}
\newcommand{\Exp}{\operatorname{Exp}}
\newcommand{\rank}{\operatorname{rank}}
\newcommand{\col}{\operatorname{col}}

\newtheorem{proposition}{Proposition}

\begin{document}
\title{RoboCompiler: Graph-Native Compilation of Closed-Chain Robots for Consistent Modeling, Control, and Simulation}

\author{Mehdi Heydari Shahna\,\orcidlink{0000-0002-1310-5392}, Joongheon Kim\,\orcidlink{0000-0003-2126-768X}, and Jouni Mattila\,\orcidlink{0000-0003-1799-4323}

\thanks{This work was supported by the grant named the Post Docs in Companies (PoDoCo) program in Finland.}
\thanks{M. H. Shahna and J. Mattila are with the Faculty of Engineering and Natural Sciences, Tampere University, Tampere, Finland; and J. Kim is with the Department of Electrical and Computer Engineering at Korea University, Seoul, Republic of Korea. (Corresponding author e-mail: mehdi.heydarishahna@tuni.fi)}
\thanks{Code: \textcolor{blue}{\url{https://github.com/Mehdi-Heydari-Shahna/RoboCompiler}}}
%\thanks{Digital Object Identifier (DOI): see top of this page.}%
}

\maketitle

\begin{abstract}
Robots with kinematic loops, coupled actuators, and changing contacts require consistent models of configuration, motion, force, and dynamics. Yet these interfaces are often reconstructed separately for control and simulation, making closure and actuation consistency difficult to maintain. This paper presents RoboCompiler, a graph-native framework that compiles a canonical mechanism graph into a shared mechanical interface. From bodies, joints, frames, inertias, and actuator ports, it constructs closure paths and analytic residual Jacobians, then assembles feasible configurations through rank-checked continuation and correction. A tangent lift maps independent velocities to full robot and task motion, while paired actuator-port maps preserve virtual work. A constraint-curvature correction extends the reduction to accelerations and projected rigid-body dynamics, including floating-base and support modes. Cycle-local evaluation, generated Jacobians, and dependency-aware reuse enable localized updates when closure inputs change. We evaluate physical loops and task-induced constraints on an industrial excavator, Unitree Go2, Franka Panda, Kangaroo, and a six-UPS Stewart platform. High-precision constrained-dynamics and independent Pinocchio checks confirm mechanical consistency; MuJoCo and Isaac Sim/PhysX executions demonstrate task performance and model reuse under native contact. For Kangaroo, compilation reduces residual-and-Jacobian evaluation time by 96.7\% and closed-loop rollout wall time by 66.8\%, with dynamics and control held fixed.
\end{abstract}
\begin{IEEEkeywords}
Closed-chain mechanisms, constrained dynamics, kinematics, parallel robots, robot modeling, simulation.
\end{IEEEkeywords}

\section{Introduction}

A robot's physical joints and actuator attachments determine its
admissible motion and how mechanical work enters the system. Closed
linkages make these relationships difficult to carry through a
computational model: a dynamics library may open a loop to form a
tree, a controller may use a different set of independent coordinates,
and a simulator may impose the same connection through native
constraints. Contact and support introduce further mode-dependent
constraints. These representations can describe one mechanism, but
their configuration, velocity, acceleration, and effort interfaces
must remain compatible. This requirement is especially consequential
for robots with multiple coupled linkages, moving actuator bodies,
and floating bases \cite{shahna2025model}.

Robot-description methods provide a foundation for representing such
systems. URDF+ adds loop joints and coupling relations to URDF and
automatically groups constrained bodies for recursive
algorithms~\cite{chignoli2024urdfplus}. Extended URDF records closure
and actuation information alongside a tree model and provides tools
to construct corresponding Pinocchio models~\cite{extended}. Both
make physical connections available to downstream computation. A
separate line of work infers a robot's connectivity and joint
geometry from proprioceptive data~\cite{morphology}. 

Methods for operating on constrained configurations address another
part of the problem. Manifold-based kinodynamic planning integrates
motion on holonomic constraint manifolds~\cite{bordalba2021kinodynamic}.
HyRoDyn combines analytical submechanism solutions with numerical
closure and provides mappings between independent and actuation
coordinates for kinematics and dynamics~\cite{kumar2020hyrodyn,
kumar2019hyrodynidetc}. The Kangaroo model develops detailed
closed-loop kinematic and actuator relationships for a floating-base
mechanism~\cite{mingohoffman2024kangaroo}. PACDM constructs closure
residuals and analytical Jacobians from ordered transformation paths,
then uses rank selection, defect homotopy, and continuation to obtain
a feasible configuration and its local differential
map~\cite{dastranj2026pacdm}. \par
Constrained dynamics and control have advanced in parallel.
Constraint embedding and GRBDA enable recursive dynamics for
mechanisms with local loops~\cite{chignoli2025propagation,grbda};
recent work develops dynamics derivatives for
constraint-embedded models~\cite{volpi}. Pinocchio provides
constrained-dynamics solvers, including recursive closed-loop forward
dynamics~\cite{carpentier2019pinocchio,sathya2026lcaba,pinocchio410},
while MuJoCo represents loop connections through equality
constraints and executes contact-rich simulations
\cite{todorov2012mujoco,mujocochangelog}. At the control level,
hierarchical optimization, redundancy resolution, geometric motion
policies, and optimal control provide ways to use constrained
kinematics and dynamics for tasks~\cite{shahna2024exponential, hqp,redundancy,pbds,parallel, shahna2025anti}.
\par Table~\ref{tab:capabilities} compares selected methods by the interfaces they expose, including closed-chain modeling, assembly, kinematic mapping, dynamics, floating-base contact, and multi-simulator execution. The comparison is intended to summarize documented interface coverage rather than to rank complete software systems with different design goals.

\begin{table*}[h!]
\centering
\begin{threeparttable}
\caption{Capabilities of closed-chain modeling and dynamics tools. \Y: yes; \Pt: partial; \N: no; \NS: not stated in the primary source; n/a: not applicable. Loops in model: closures are part of the model/graph description. Auto.\ cut sel.: spanning tree and cut joints chosen automatically from a general body--joint graph. Assembly: built-in position-level closure solve. Tangent map: explicit map from independent to full velocities. Port maps: explicit actuator-port velocity map with its dual (virtual-work-consistent) effort map; \Pt\ = joint-space torques, gears or couplings only. FD: R = reduced/projected, K = Lagrange-multiplier/KKT, Rec = recursive, S = soft (regularized) constraints; A/N = analytic/numerical $\dot J\dot q$. Float.\ base + contact: floating-base models with contact or support constraints (bilateral or unilateral). Multi-sim.: execution or validation of the same model in external simulators or dynamics libraries (\Y: several; \Pt: one).}
\label{tab:capabilities}
\footnotesize
\setlength{\tabcolsep}{3pt}
\begin{tabular}{@{}>{\raggedright\arraybackslash}p{3.1cm}cc>{\raggedright\arraybackslash}p{2.35cm}ccccccc@{}}
\toprule
Work & \makecell{Loops in\\model} & \makecell{Auto.\\cut sel.} & Closure types & \makecell{Assem-\\bly} &
\makecell{Tangent\\map} & \makecell{$\dot J\dot q$\\term} & \makecell{Port\\maps} & FD &
\makecell{Float.\ base\\+ contact} & \makecell{Multi-\\sim.} \\
\midrule
Bordalba et al.~\cite{bordalba2021kinodynamic} & \NS & \NS & holonomic $\Phi(q)=0$ & \Y & \Pt\tnote{a} & \NS & \NS & manifold ODE & \NS\tnote{b} & \NS \\
HyRoDyn~\cite{kumar2020hyrodyn,kumar2019hyrodynidetc} & \Pt & \N & submechanism library\tnote{c} & \Y & \Y & \Y\,(A) & \Y & R & \NS & \NS \\
Kangaroo model~\cite{mingohoffman2024kangaroo} & \Y & \N & relative pose, selected axes & \Y & \Y & \Y\,(A) & \Y & ID only\tnote{d} & \Y & \Pt \\
URDF+~\cite{chignoli2024urdfplus} & \Y & \N & URDF joint types, linear couplings & \NS & \Pt & \NS & \Pt & n/a\tnote{e} & \NS & \NS \\
Constraint embedding / GRBDA~\cite{chignoli2025propagation,grbda} & \Y\tnote{f} & \N & explicit $G$, implicit $\phi(q)$ & \Pt & \Y & \Y\,(A) & \Pt & Rec & \Pt & \Pt \\
LCABA~\cite{sathya2026lcaba} / Pinocchio 4.1~\cite{carpentier2019pinocchio,pinocchio410} & \Pt & \Pt\tnote{g} & point (3D), weld (6D) & \Pt\tnote{h} & \N & \Y\,(A)\tnote{i} & \Pt & K, Rec & \Y & \Pt\tnote{j} \\
MuJoCo 3.3.7 / 3.14~\cite{todorov2012mujoco,mujocochangelog} & \Y & \N & point, weld, scalar joint/tendon & \Pt\tnote{k} & \N & \N\,/\,\Pt\tnote{l} & \Y & S & \Y & n/a \\
PACDM~\cite{dastranj2026pacdm} & \Pt & \N & SE(3) path-pose equality & \Y & \Y & \NS & \NS & \N\ (kinematics) & \NS & \Pt \\
\midrule
RoboCompiler (this work) & \Y & \Pt\tnote{m} & point, revolute, universal; welds & \Y & \Y & \Y\,(N)\tnote{n} & \Y & R & \Y & \Y \\
\bottomrule
\end{tabular}
\begin{tablenotes}[flushleft]\footnotesize
\item[a] Orthonormal basis of the state-manifold tangent space at each chart centre (local chart coordinates), not an explicit $\dot q$ map.
\item[b] All examples are fixed-base; unilateral contact is listed as future work.
\item[c] Closed-form submechanism types (e.g., 1-RRPR, 2-SPU+1U, 6-UPS) plus a numerical fallback; users annotate the submechanisms.
\item[d] Inverse dynamics via Lagrange multipliers; no forward-dynamics algorithm is stated.
\item[e] Model format; its output is intended for recursive constraint-embedding algorithms.
\item[f] Through URDF+ in the GRBDA library ($\geq$v2.1.0); the T-RO paper itself uses a cluster-graph formalism.
\item[g] The SDF parser keeps one parent joint per multi-parent link in the tree and converts the others into constraints.
\item[h] No library routine; the official examples hand-code an inverse-geometry loop.
\item[i] Analytic drift $\gamma_{c,T}$. In Pinocchio 4.1.0 \texttt{lcaba}, the Baumgarte term is active only for 6D constraints.
\item[j] A Pinocchio unit test compares closed-chain MJCF kinematics with MuJoCo; the LCABA benchmarks stay inside Pinocchio.
\item[k] No assembly routine. With anchors, the constraint is assumed satisfied at \texttt{qpos0}; with sites, the sites snap together at the beginning of simulation.
\item[l] Not computed in 3.3.7; added for \texttt{connect}/\texttt{weld} in 3.7.0; still absent for the \texttt{tendon} rows in our 3.14.0 tests.
\item[m] Tree and revolute cuts generated from physical records in the excavator compiler; the Stewart compiler generates the physical tree for each spherical tree-joint choice; the Kangaroo compiler takes declared cuts.
\item[n] Analytic closure Jacobian; $\dot J\dot q$ by a centered directional finite difference of it.
\end{tablenotes}
\end{threeparttable}
\end{table*}

\begin{figure}[h!]
\centering
\includegraphics[width=\columnwidth]{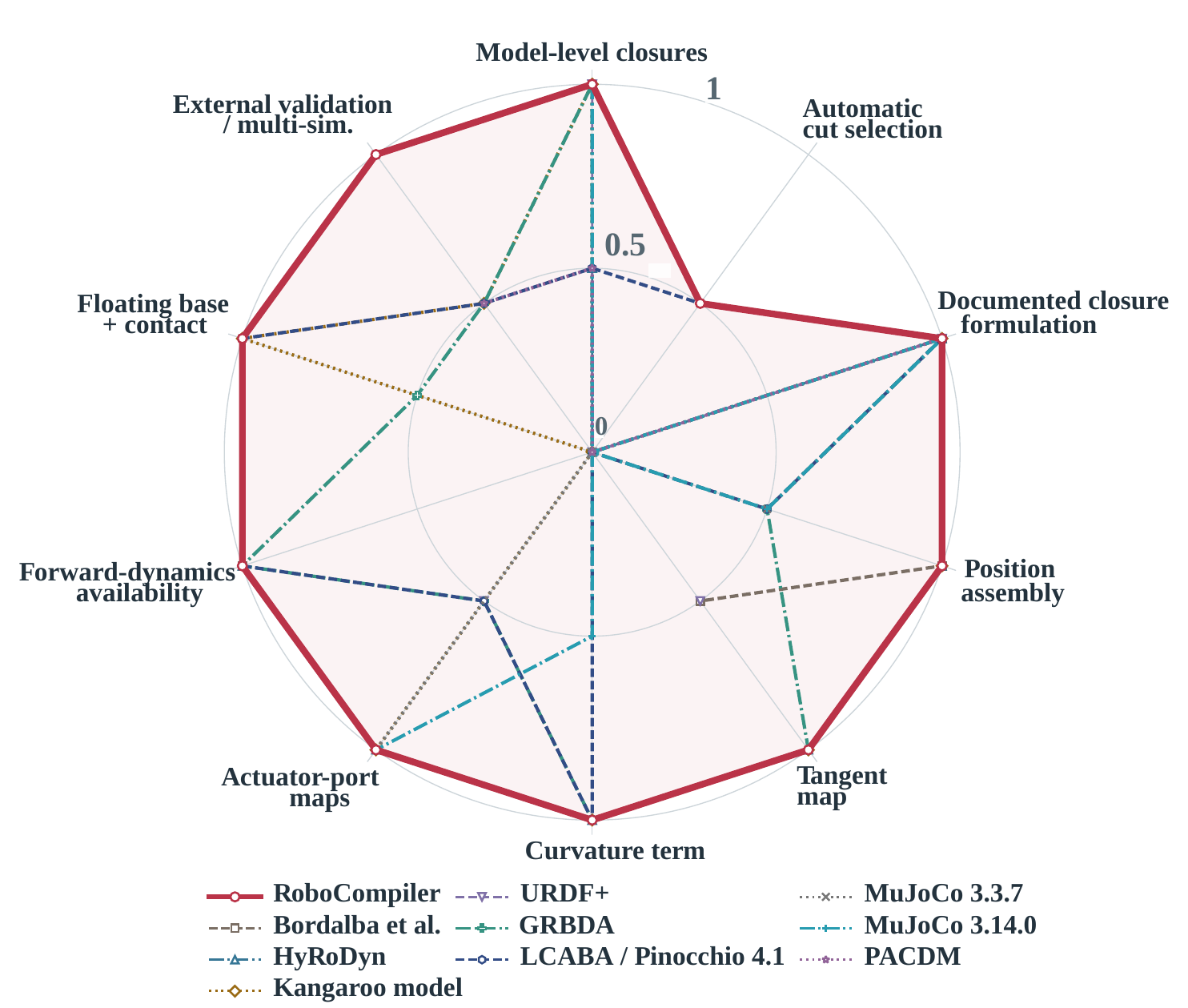}
\caption{Categorical interface coverage from
Table~\ref{tab:capabilities}, with RoboCompiler highlighted.
Scores 1, 0.5, and 0 encode documented, partial, and absent
support; not-stated and inapplicable categories remain gaps.}
\label{fig:rc_capabilities_overlay}
\end{figure}

\begin{figure*}[h!]
\centering
\includegraphics[width=\textwidth,height=0.65\textheight,
keepaspectratio]{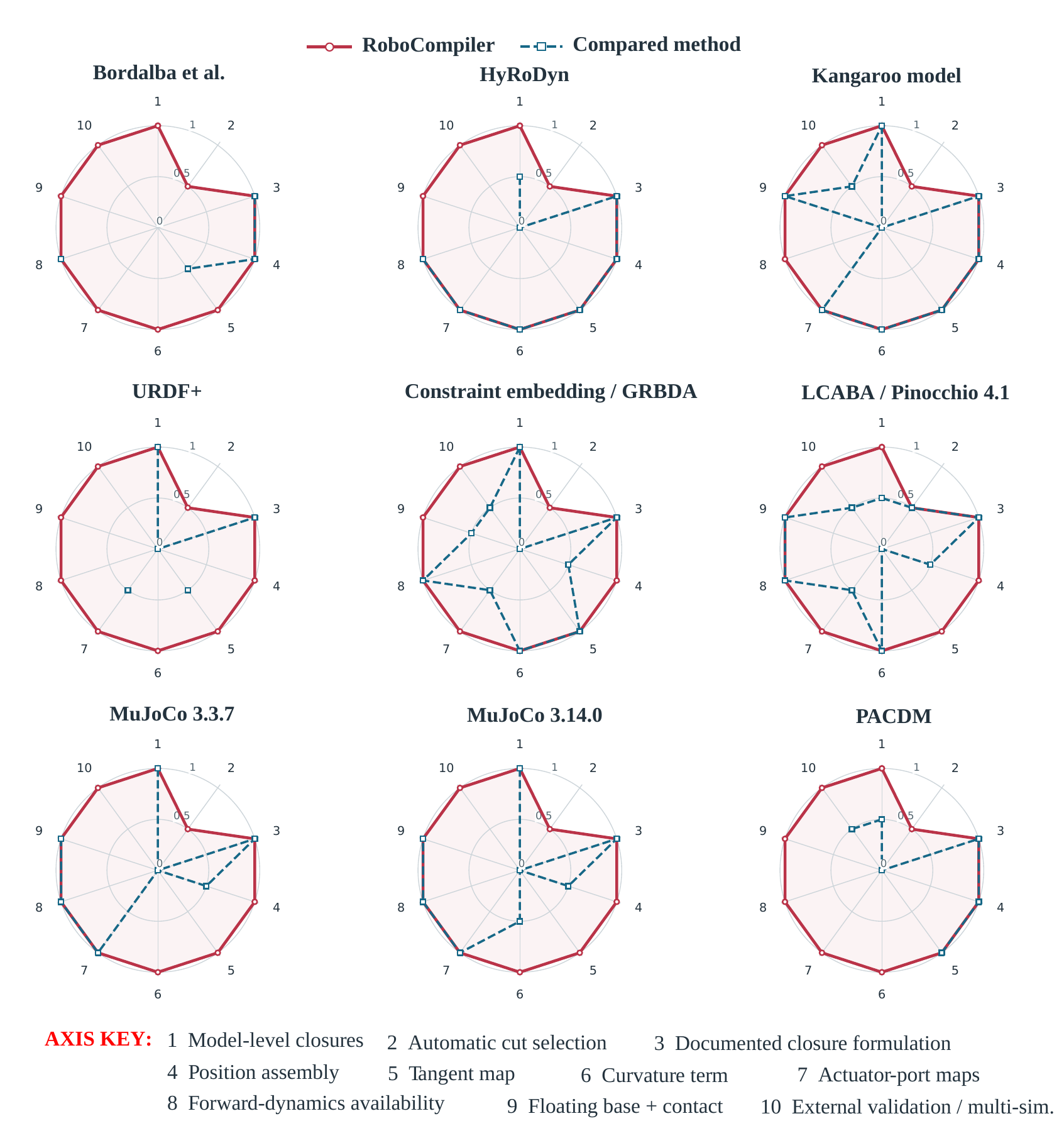}
\caption{Pairwise views of the same categorical entries in
Table~\ref{tab:capabilities}: RoboCompiler (solid red) and each
listed method (dashed blue). A missing or inapplicable entry is
unscored.}
\label{fig:rc_capabilities_panels}
\end{figure*}

Taken together, these advances address major parts of closed-chain model description, constrained assembly, kinematic reduction, dynamics, simulation, and control. They leave, however, an important integration question: how can these computational views be generated from the same physical mechanism description while preserving assembly, motion, acceleration, actuation, and mechanical-power consistency across them?

A physical
description, a feasible closure solution, and a constrained-dynamics
solver must refer to the same assembly branch and actuator
coordinates if their outputs are to be used together. Changing the
computational tree must preserve the restrictions imposed by each
physical connection. Changing independent coordinates must carry the
task differentials and reduced inertia with them. Actuator velocities
and efforts must transform as a power-conjugate pair, and the
acceleration map must retain constraint curvature. Agreement of
coordinate vectors or successful model import alone does not
establish all of these relationships.

We develop RoboCompiler to address this integration problem. RoboCompiler is a graph-native compiler built around a canonical mechanism graph (CMG) whose body, joint, frame, inertia, and actuator-port records retain their physical identities. Starting from the physical mechanism records, the compiler selects tree and cut assignments or accepts declared cuts, generates ordered closure paths and coordinate-dependency records, and provides the path structure required by PACDM.

PACDM continuation and correction establish a feasible assembly branch, after which rank checks and analytical closure Jacobians provide a tangent lift from independent to full velocities. A constraint-curvature term extends the lift to accelerations, while paired port-velocity and effort maps preserve virtual work. These quantities feed reduced rigid-body dynamics and task-level computation, and backend-specific views support independent dynamics checks and native execution in external simulators. Cycle-local evaluation and reuse of unaffected dependency modules reduce repeated closure computation while preserving the accepted PACDM solution.

In this sense, RoboCompiler is complementary to robot-description formats, constrained-dynamics libraries, and simulators: rather than replacing them, it compiles a shared physical mechanism description into mechanically consistent interfaces that these tools can use. Table~\ref{tab:capabilities} and Figs.~\ref{fig:rc_capabilities_overlay} and~\ref{fig:rc_capabilities_panels} summarize the resulting interface coverage: among the compared methods, only RoboCompiler is documented to combine point, revolute, and universal cuts and position-level assembly with an explicit independent-to-full tangent map, virtual-work-dual actuator-port maps, reduced forward dynamics, floating-base/contact workflows, and multi-simulator validation.

The contributions of this work are as follows:
\begin{itemize}
    \item We introduce a graph-native compilation procedure that
    transforms physical mechanism and task records into ordered
    closure paths, dependency modules, rank-checked coordinate
    partitions, and connected interfaces for configuration, motion,
    actuation, and dynamics.

    \item From a shared assembled configuration, we generate the
    tangent lift, a constraint-curvature correction, and
    virtual-work-dual actuator-port maps. Together, these form a
    common independent-coordinate interface for motion control and
    reduced inverse and forward dynamics of fixed-base robots and of floating-base robots under ideal support modes.

    \item We generate cycle-local analytical closure Jacobians and
    dependency-aware update schedules from the graph. Selective
    evaluation and reuse reduce repeated work during continuation
    and local task changes while preserving the underlying PACDM
    closure and correction equations.

    \item We evaluate the interfaces on three physically
    closed-chain mechanisms and two robots with task or actuator
    coupling. Independent constrained-dynamics checks, matched
    computational ablations, and executions in MuJoCo, Pinocchio,
    and Isaac Sim demonstrate mechanical consistency, computational benefit, and reuse across environments.
\end{itemize}

The complete code for all five robots and their MuJoCo, Pinocchio, and Isaac Sim/PhysX workflows is available open source at \textcolor{blue}{\url{https://github.com/Mehdi-Heydari-Shahna/RoboCompiler}}.

\section{Physical Description and Compilation}
\label{sec:graph}

\subsection{Physical Data and Velocity Conventions}

Let $\mathcal G=(\mathcal B,\mathcal J)$ be a body--joint multigraph (principal quantities are listed in Table~\ref{tab:notation}).
Each vertex denotes a rigid-body instance with geometry,
attachment frames, and inertial parameters. Each edge denotes a
physical joint with a permitted motion model and limits.
Repeated parts remain separate instances.
Actuator ports, transmissions, and their limits are separate physical records, so the choice of independent coordinates is decoupled from actuation.

For an oriented joint $j=(b,c)$, define
\begin{equation}
    X_j(q_j)=A_j Q_j(q_j)B_j^{-1}\in\SE,
    \label{eq:edge}
\end{equation}
where $A_j={}^{b}T_{J_b}$ and $B_j={}^{c}T_{J_c}$ locate the two
attachment frames, and $Q_j={}^{J_b}T_{J_c}$ describes the allowed
relative motion. Thus $X_j={}^{b}T_c$ maps coordinates from frame
$c$ to frame $b$. Reverse traversal uses $X_j^{-1}$. Fixed,
revolute, prismatic, spherical, and universal joints fit this
description when their local parameterizations and motion maps
are supplied. Additional holonomic couplings, for example a prescribed gear relation, enter as $\gamma(q)=0$, with their equations and derivatives provided as physical inputs.

The augmented configuration $q\in\mathcal Q$ contains all joint
coordinates, including those on cut edges, and any free-root
poses. Its generalized velocity is $v\in R^n$, with
\begin{equation}
    \dot q=N(q)v.
    \label{eq:N}
\end{equation}
Here $N$ is a full-column-rank local map from generalized
velocities to configuration derivatives; in particular, a
quaternion derivative is not an angular velocity. Spatial vectors
use angular-first ordering, and $\Ad_T$ is the corresponding
adjoint. A numerical increment $\delta v$ is applied through a
manifold integration operation $q\oplus\delta v$.

\begin{table}[h!]
    \caption{Principal quantities and their roles.}
    \label{tab:notation}
    \centering
    \small
    \begin{tabular}{@{}p{.22\columnwidth}p{.72\columnwidth}@{}}
        \toprule
        Symbol & Definition \\
        \midrule
        $q\in\mathcal Q$, $v\in R^n$
        & Augmented configuration and generalized velocity \\

        $\rho\in R^s$, $C$
        & Full closure residual; $C=D\rho\,N$ \\

        $r$, $d=n-r$
        & Closure rank and independent velocity dimension \\

        $u\in R^d$, $E$
        & Independent speeds and lift $v=Eu$ \\

        $v_T$, $E_T$
        & Tree speeds and their lift $v_T=E_Tu$ \\

        $x$, $J$
        & Task coordinates and reduced task Jacobian \\

        $\ell$, $L$, $f$
        & Actuator displacement, reduced transmission, effort \\

        $M_r$, $h_r$
        & Reduced inertia and bias terms \\

        $W$, $g$
        & Reduced physical metric and objective differential \\
        \bottomrule
    \end{tabular}
\end{table}

\subsection{From Physical Edges to Closure Comparisons}

Choose a spanning forest $\mathcal T$ and a reference body for
each connected component. A computational reference does not
impose a physical fixation: a free-root pose remains a state
variable. The number of non-tree edges is
\begin{equation}
    \mu=|\mathcal J|-|\mathcal B|+\kappa,
\end{equation}
where $\kappa$ is the number of connected components. This cycle
rank is not the number of independent scalar constraints.

Each chord $e=(b,c)$ defines a unique ordered forest path $P_e$
between its endpoints. Let $\sigma_{ej}\in\{-1,1\}$ specify
traversal direction. Form
\begin{equation}
    T_e^+=\prod_{j\in P_e}^{\longrightarrow}X_j^{\sigma_{ej}},
    \qquad
    T_e^-=X_e,
    \qquad
    \Delta_e=(T_e^-)^{-1}T_e^+.
    \label{eq:paths}
\end{equation}
Closure requires $\Delta_e=I$. Retaining $q_e$ allows the cut
joint to execute its physical relative motion. Replacing this
equality by a weld after deleting $q_e$ would generally change
the mechanism. Joint-specific constraint elimination is possible,
but it must preserve the same admissible relative transformations.

\textit{Completeness of the edge construction}.
For fixed component-root poses and specified joint configurations,
forest propagation together with $\Delta_e=I$ for every chord
satisfies every physical edge relation in $\mathcal G$.

\textit{Proof}. Forest propagation assigns body poses satisfying all tree edges.
For a chord, $\Delta_e=I$ equates the propagated endpoint transform
with its physical joint transform. Every edge is either a tree
edge or a chord. Conversely, body poses satisfying every edge
necessarily satisfy these comparisons.

This statement establishes structural edge completeness. Feasibility, including the supplied couplings $\gamma(q)=0$, is established by the closure solve of Section~\ref{sec:closure}, and scalar independence by its rank analysis.
Figure~\ref{fig:compilation} summarizes the construction.

\subsection{Dependency Modules and Compiler Output}

A candidate local partition distinguishes independent and
dependent variables. For each closure comparison, let
$\mathcal D_e$ contain the dependent variable blocks occurring in
its ordered factors. The supplied coupling residuals are included
through their own dependent supports. An inverted variable index
connects residual blocks sharing a dependent variable; its
connected components define candidate modules. Modules may share
prescribed independent variables. Their dependent variables are
disjoint, permitting separate closure solves when those prescribed
values are fixed.

This grouping is conservative because structural dependencies can
cancel algebraically. An independent-only constraint is retained: it must be locally redundant, or the partition is revised. Likewise, an unconstrained variable is never assigned to the dependent set. The rank conditions in
Section~\ref{sec:closure} decide whether the candidate partition
yields a valid reduction.

Each emitted comparison records ordered factors, traversal signs,
coordinate indices, and source-edge identifiers. Different
comparisons in a dependency module may have different endpoints.
PACDM is therefore applied to each endpoint-matched pair and then assembled in common module coordinates. When a module forms a single common-endpoint path family, $N_p$ paths need only $N_p-1$ pairwise comparisons~\cite{dastranj2026pacdm}.

Forest construction and explicit path extraction require
$O(|\mathcal B|+|\mathcal J|+L_P)$ operations, where
$L_P=\sum_e(|P_e|+1)$. With union--find, grouping $N_c$
closure/coupling blocks costs $O(N_c+S_D\alpha(N_c))$, where $S_D$
is their total dependent-support size. These bounds cover the structural compilation stage; numerical rank analysis and closure resolution (Section~\ref{sec:closure}) are separate stages.
Listing~\ref{lst:kangaroo} shows the Kangaroo input. It contains
only physical records; the compiler rejects inputs that
pre-author paths, partitions, modules, or coordinate orders, and
generates them instead. For Kangaroo, it yields 76 joint and 64
chart coordinates in six loop modules. Each module is corrected
separately and is skipped when its motor inputs are unchanged.
Function and record names are those of the implementation;
variable names follow the notation of this paper.

\begin{lstlisting}[float=tp, label={lst:kangaroo},
  caption={Abridged Kangaroo input and compiler calls. The input
  holds physical records only; \texttt{step} returns the joint
  configuration and the joint rows of the lift $E$
  in~\eqref{eq:lift}.}]
physical = {  # Kangaroo records, abridged
  "floating_root": "base_link",
  "bodies": [...],  # 78, with mass and inertia
  "joints": [...],  # 77 edges, frames, limits
  "loop_cuts": [    # 16 point + 8 universal
    {"id": "left_hip_yaw_closed",
     "type": "point_coincidence",
     "body1": "left_yaw_rod",
     "body2": "left_hip_yaw", ...},
    {"id": "left_hip_differential_2a",
     "type": "universal", ...}, ...],
  "actuators": [    # 12 actuator ports
    {"id": "leg_left_1_actuator_torque",
     "joint_id": "leg_left_1_motor",
     "gear": 1.0,
     "force_bounds_N": [-2000.0, 2000.0]}, ...],
  "drive": {...}, "sole_sites": [...]}
  
comp = compile_graph(physical)
# generated: tree, cut paths, charts, 6 modules
solver = ModularSolver(comp, x0)  # x0: assembled
q, E, info = solver.step(q_motor)  # 12 motors
# E (joint rows); modules whose inputs are
# unchanged are skipped
\end{lstlisting}

\begin{figure*}[h!]
\centering
\includegraphics[
width=0.85\textwidth,
height=0.450\textheight,
keepaspectratio
]{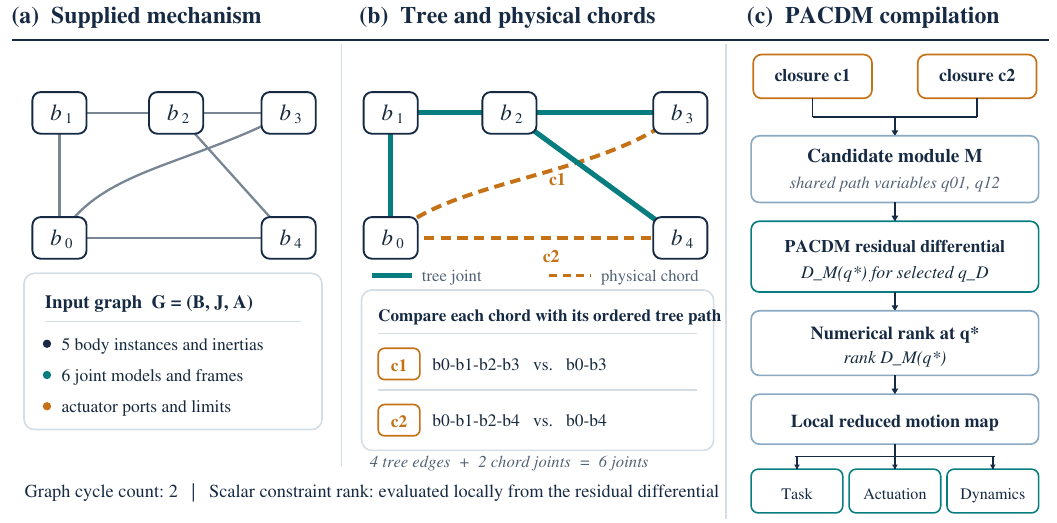}
\caption{Compilation from a supplied physical mechanism graph.
    (a) Body instances, joint attachment frames, and actuator ports
    define the input. (b) A five-body, six-joint connected graph
    admits a four-edge spanning tree and two chords; each chord is
    compared with its ordered tree path. (c) Comparisons sharing
    dependent variables form a candidate module. PACDM supplies
    the residual differential, while numerical rank determines the
    local reduction. Dashed chords remain physical joints. Graph
    cycle count and scalar constraint rank are distinct.}
\label{fig:compilation}
\end{figure*}

\section{PACDM Closure and Local Reduction}
\label{sec:closure}

\subsection{Path-Assembled Differentials}

Following PACDM~\cite{dastranj2026pacdm}, define
\begin{equation}
    \begin{gathered}
        \rho_e(q)=\Log(\Delta_e(q))^\vee,
        \qquad
        \rho(q)=\col(\col_e\rho_e(q),\gamma(q)),\\
        C(q)=D\rho(q)N(q).
    \end{gathered}
    \label{eq:closure}
\end{equation}
The logarithm is evaluated on a fixed local branch whose domain is enforced as an acceptance condition (Section~\ref{subsec:rank}). The feasible configuration set is
$\mathcal C=\{q:\rho(q)=0\}$.

For an ordered path
$T^\sigma=F_1^\sigma\cdots F_k^\sigma$, define
$Z_i^\sigma=F_1^\sigma\cdots F_{i-1}^\sigma$. Local motion maps
$\Xi_i^\sigma$ satisfy
$(\dot F_i^\sigma(F_i^\sigma)^{-1})^\vee=\Xi_i^\sigma v$.
Product differentiation gives
\begin{equation}
    (\dot T^\sigma(T^\sigma)^{-1})^\vee=H^\sigma v,
    \qquad
    H^\sigma=\sum_i\Ad_{Z_i^\sigma}\Xi_i^\sigma.
\end{equation}
Consequently, the differential of each residual is
\begin{equation}
    C_e=J_l^{-1}(\rho_e)
    \Ad_{(T_e^-)^{-1}}(H_e^+-H_e^-),
    \label{eq:pacdm}
\end{equation}
where $J_l$ is defined by
$(\dot\Delta\Delta^{-1})^\vee=J_l(\rho)\dot\rho$.
Equation~\eqref{eq:pacdm} retains the logarithm differential away
from closure; substituting $J_l^{-1}=I$ there is generally
incorrect. Rigid factors contribute through the adjoint
transports even when their own derivatives vanish. The supplied
coupling rows contribute $D\gamma(q)N(q)$ to $C$.

\subsection{Rank, Recovery, and Continuation}
\label{subsec:rank}

Consider a neighborhood of a feasible configuration where
$\rank C=r$ is constant. A permutation $\Pi$ partitions the
velocity into $\Pi v=\col(v_a,v_p)$, with $v_p\in\R^r$. Choose a
fixed row selector $S\in\R^{r\times s}$ such that
\begin{equation}
    \begin{gathered}
        C\Pi^\top=[C_a\ C_p],
        \qquad
        \det R_p\ne0,\\
        R_a=SC_a,
        \qquad
        R_p=SC_p.
    \end{gathered}
    \label{eq:partition}
\end{equation}
The selected rows must span the full row space locally. Then
\begin{equation}
    v=E(q)u,
    \qquad
    E=\Pi^\top
    \begin{bmatrix}
        I_d\\
        -R_p^{-1}R_a
    \end{bmatrix},
    \qquad
    CE=0,
    \label{eq:lift}
\end{equation}
with $u=v_a$ and $d=n-r$. These are independent speeds, not
necessarily derivatives of a global coordinate vector. The
corresponding position solve is performed in compatible local
configuration charts.

Rank-revealing factorizations select rows and candidate dependent
blocks using declared length and velocity scales. Linear systems
are solved without explicitly forming $R_p^{-1}$. A singular $R_p$ with unchanged $\rank C$ calls for a new velocity partition, that is, a computational reconfiguration (Section~\ref{subsec:reconfig}); a change in $\rank C$ itself is a distinct, mechanism-level event. The complete
residual is checked after every accepted solve. 

For acquisition, PACDM uses an artificial defect~\cite{dastranj2026pacdm}.
Given $q^0$, set
\begin{equation}
    \begin{split}
        D_e(\eta)
        &=\Exp\big((1-\eta)\Log\Delta_e(q^0)\big),\\
        \rho_e^\eta(q)
        &=\Log\big(D_e(\eta)^{-1}\Delta_e(q)\big)^\vee,
        \qquad \eta\in[0,1].
    \end{split}
    \label{eq:homotopy}
\end{equation}
Coupling residuals use
$\gamma^\eta(q)=\gamma(q)-(1-\eta)\gamma(q^0)$.
The initial state then satisfies the artificial problem at
$\eta=0$. Continuation removes the defect while prescribed
independent variables are held fixed. Acquisition uses its own
independent row set and requires a regular feasible continuation
path. At $\eta=1$, all physical residuals are checked and the
physical reduction is recomputed; acquisition and physical row
selectors need not coincide.

During motion, the lift predicts
$q^-_{k+1}=q_k\oplus(hE(q_k)u_k)$, followed by a local closure
correction. Recovery is invoked when direct correction fails.
Joint limits, the selected assembly branch, and the logarithm
domain remain acceptance conditions.

\section{Motion, Actuation, and Dynamics Interfaces}
\label{sec:interfaces}

\subsection{Task and Actuator Maps}

For a task $x=f_x(q)\in\R^m$ expressed in a fixed local task chart,
let $J_x=D f_x\,N$. A differentiable actuator displacement map
$\ell=f_\ell(q)\in\R^{n_f}$ similarly gives
$J_\ell=D f_\ell\,N$. The compiled interfaces are
\begin{equation}
    \dot x=Ju,\quad J=J_xE,
    \qquad
    \dot\ell=Lu,\quad L=J_\ell E.
    \label{eq:interfaces}
\end{equation}
Actuator displacement can be a cylinder extension or a local motor
angle. Transmissions without a holonomic displacement map are included through an explicitly supplied velocity map and its effort dual.

If $f\in\R^{n_f}$ is the actual effort conjugate to $\dot\ell$,
virtual work gives
\begin{equation}
    \tau_r=L^\top f,
    \qquad
    u^\top\tau_r=\dot\ell^\top f.
    \label{eq:power}
\end{equation}
A commanded effort and an actual effort are different quantities
when drive dynamics are present. Such dynamics enter through
specified states $z$ and commands $c$, for example
$\dot z=f_z(z,c,q,v)$ and $f=f_f(z,c,q,v)$. These constitutive laws are supplied as model inputs alongside the mechanism graph. Friction and transmission losses enter as additional terms; Eq.~\eqref{eq:power} holds at the declared mechanical ports.

\subsection{Physical Inertia and Reduced Dynamics}

Closure coordinates attached to cut joints do not create
additional bodies or mass. Let $q_T=\pi_T(q)$ denote the tree
configuration and let $v_T=K_T(q)v$ be its differential velocity
map. With $N_T$ denoting the tree configuration-rate map,
compatibility requires
$D\pi_T(q)N(q)=N_T(q_T)K_T(q)$. Define $E_T=K_TE$.
The tree model contains each physical body's inertia exactly once.
Its applied effort is decomposed as
$\tau_T=\tau_{a,T}+\tau_{\mathrm{ext},T}$, with
$E_T^\top\tau_{a,T}=L^\top f$ and
$\tau_{r,\mathrm{ext}}=E_T^\top\tau_{\mathrm{ext},T}$.
Its constrained equations are
\begin{equation}
    M_T\dot v_T+h_T=\tau_T+J_{c,T}^\top\lambda,
    \qquad
    J_{c,T}E_T=0,
\end{equation}
where $\lambda$ represents ideal internal closure reactions and
$h_T$ includes the tree's bias terms~\cite{featherstone}. Since
$v_T=E_Tu$, projection yields
\begin{equation}
    \begin{split}
        M_r\dot u+h_r&=L^\top f+\tau_{r,\mathrm{ext}},\\
        M_r&=E_T^\top M_TE_T,\\
        h_r&=E_T^\top\big(h_T+M_T\dot E_Tu\big).
    \end{split}
    \label{eq:dynamics}
\end{equation}
The derivative $\dot E_T$ is taken along the physical trajectory
and includes any variation of $K_T$. Omitting $M_T\dot E_Tu$
changes the acceleration model. External forces, including
contact forces when present, remain in $\tau_{r,\mathrm{ext}}$
rather than disappearing with internal loop reactions.

Cut-joint actuator efforts are applied through physical attachment wrenches or an equivalent local transmission, so every effort acts on retained physical coordinates.

Positive definiteness of $M_r$ requires a positive physical
inertia on the image of $E_T$ and full column rank of that lift.
Auxiliary gauge freedoms must be removed before an inertia
inverse is used. Moving actuator bodies retain their own inertia, and supplied rotor or reflected-inertia models are counted once, consistently with the tree model. These expressions follow standard constraint embedding~\cite{chignoli2025propagation,volpi}; RoboCompiler supplies the compatible maps $E_T$ and $L$ and the curvature term automatically from the physical graph.
Listing~\ref{lst:panda} shows these calls for Panda. The equal
finger displacements enter as the affine coupling
$\gamma(q)=q_{\mathrm{f}2}-q_{\mathrm{f}1}=0$, and the gripper
motor is a single actuator port acting on both fingers.

\begin{lstlisting}[float=tp, label={lst:panda},
  caption={Abridged Panda input and reduced-dynamics calls. Here
  $v_T$ and $c_T$ are the tree rows of $v=Eu$ and of the curvature
  term, $f$ collects the actuator efforts, and $w$ is the tool
  wrench with Jacobian $J_w$; the last line evaluates the
  virtual-power identity~\eqref{eq:power}.}]
physical = {  # Panda records, abridged
  "joints": [...],  # 11 tree edges, 9 moving
  "affine_couplings": [{"master": "finger_joint1",
     "slave": "finger_joint2",
     "multiplier": 1.0, "offset": 0.0}],
  "actuators": [...,  # 8 actuator ports
     {"id": "actuator8", "transmission": {
       "finger_joint1": 0.5,
       "finger_joint2": 0.5}}],
  "tool": {"body": "hand", "T_body_tool": ...},
  "redundancy_joint": "joint3", ...}

comp = compile_graph(physical)  # q_f2 = q_f1
B = comp.cmg["actuation"]["moment_matrix"]
g = GeneratedTaskGraph(comp)    # 6-D tool target
# lift E, velocity v = E u, curvature term c:
E, v, c, _ = curvature(g, g.lift(q), u)
ET = g.physical_map(E)          # tree rows of E
tau = B @ f + Jw.T @ w          # ports and wrench
Mr = ET.T @ MT @ ET             # reduced inertia
ud = solve(Mr, ET.T @ (tau - hT - MT @ cT))
# virtual-power defect:
defect = tau @ vT - (ET.T @ tau) @ u
\end{lstlisting}

\section{Intrinsic Velocity Selection and Reconfiguration}
\label{sec:control}

\subsection{A Physical Secondary Objective}

At a regular task value, the task fiber is
\begin{equation}
    \mathcal F_x=\{q\in\mathcal C:f_x(q)=x\},
    \qquad
    \dim\mathcal F_x=d-\rank J.
\end{equation}
Passive loop coordinates do not increase this dimension.
Let $\Phi(q)$ be a differentiable physical preference, and define
$g=E^\top N^\top D\Phi^\top$, so that $\dot\Phi=g^\top u$.
For instance, with finite actuator stroke limits and
$s_i=(\ell_i-\ell_i^-)/(\ell_i^+-\ell_i^-)$, an interior preference
is
\begin{equation}
    \Phi(q)=-\sum_i\beta_i\log\big(s_i(1-s_i)\big),
    \qquad
    \beta_i>0.
    \label{eq:objective}
\end{equation}
This objective penalizes proximity to the specified stroke limits.

Choose a positive physical motion metric whose reduced matrix is
$W\succ0$. One admissible choice is $W=M_r$ when its definiteness
conditions hold. For task error $e_x=x-x_d$ in the chosen task
chart, prescribe $w=\dot x_d-K_xe_x$, where $K_x\succ0$.
In the compiled coordinates, the constrained velocity problem (Fig.~\ref{fig:equivalence}) is
\begin{equation}
    \begin{aligned}
        u^*=\arg\min_u\quad
        &\tfrac12 u^\top Wu+k_\Phi g^\top u\\
        \text{subject to}\quad
        &Ju=w,\qquad Au\le b,
    \end{aligned}
    \label{eq:qp}
\end{equation}
with $k_\Phi>0$ and physically specified bounds. For example,
actuator speed bounds use $A=\col(L,-L)$ and
$b=\col(\dot\ell^+,-\dot\ell^-)$. Other kinematic restrictions may
be appended with their correct reduced maps. Feasibility and
$W\succ0$ imply a unique solution. An infeasible task is rejected explicitly or handled by a declared relaxation policy, so the task equality is never softened silently.

Without inequalities and with full row rank $J$,
Eq.~\eqref{eq:qp} reduces to
\begin{equation}
    \begin{split}
        J_W^\#&=W^{-1}J^\top(JW^{-1}J^\top)^{-1},\\
        u^*&=J_W^\#w-k_\Phi(I-J_W^\#J)W^{-1}g.
    \end{split}
\end{equation}
This is weighted redundancy resolution~\cite{hqp,redundancy,pbds}, posed here directly in the compiled independent coordinates with the reduced physical metric $W$. Under exact feasible
execution, $\dot e_x=-K_xe_x$. At a stationary task with $w=0$
and no inequalities,
\begin{equation}
    \dot\Phi=-k_\Phi g^\top
    \big(
    W^{-1}
    -W^{-1}J^\top(JW^{-1}J^\top)^{-1}JW^{-1}
    \big)g
    \le0.
\end{equation}
The physical preference $\Phi$ is therefore non-increasing along the resulting self-motion. Floating-base or otherwise underactuated systems additionally require a dynamic/contact feasibility layer to realize this velocity with admissible actuator efforts and contact forces.

\subsection{Consistency Under Representation Changes}

Two representations are equivalent here when they describe the
same physical configurations, body velocities, actuator ports,
and task on an overlapping regular branch. Use $\mathcal V=Bu$
for the stacked physical body velocities in common frames.
After removing auxiliary gauge freedoms, assume both
representations give full-rank bases of the same tangent space.
There is then an invertible $F(q)$ such that
\begin{equation}
    \widetilde B=BF,
    \qquad
    u=F\widetilde u.
\end{equation}
The corresponding reduced quantities must satisfy
\begin{equation}
    \begin{gathered}
        \widetilde J=JF,\quad
        \widetilde L=LF,\quad
        \widetilde A=AF,\quad
        \widetilde W=F^\top WF,\\
        \widetilde g=F^\top g,
        \qquad
        \widetilde\tau_r=F^\top\tau_r.
    \end{gathered}
    \label{eq:covariance}
\end{equation}
Task coordinates, the physical objective, and bounds are held
fixed. Changing their physical meaning is not a
reparameterization.

\begin{proposition}[Motion and power consistency]
Under the preceding equivalence assumptions, with a feasible
problem~\eqref{eq:qp} and $W\succ0$, the transformed problem has a
unique solution satisfying $u^*=F\widetilde u^*$. Both
representations yield the same body velocities, task velocity,
actuator velocity, and mechanical power.
\end{proposition}

Proof.
The bijection $u=F\widetilde u$ maps the feasible sets onto each
other by Eq.~\eqref{eq:covariance}, and preserves both terms of the
objective. Strict convexity gives the stated correspondence of
minimizers. Substitution into $Bu$, $Ju$, and $Lu$ preserves
physical velocities. Finally,
$\widetilde u^\top\widetilde\tau_r
=u^\top\tau_r=f^\top\dot\ell$.

The proposition carries coordinate covariance through all generated task, effort, and inequality interfaces and specifies the transformation contract~\eqref{eq:covariance}. In particular, replacing a transformed
physical metric by the identity in every chart generally changes
the optimization problem.

For dynamics, configuration-dependent $F$ also produces an
acceleration term:
\begin{equation}
    \widetilde M_r=F^\top M_rF,
    \qquad
    \widetilde h_r=F^\top(h_r+M_r\dot F\widetilde u).
    \label{eq:dynamicchange}
\end{equation}
Therefore, matching reduced mass matrices alone is insufficient
to establish equivalent dynamic equations.

\begin{figure*}[h!]
\centering
\includegraphics[
width=0.85\textwidth,
height=0.450\textheight,
keepaspectratio
]{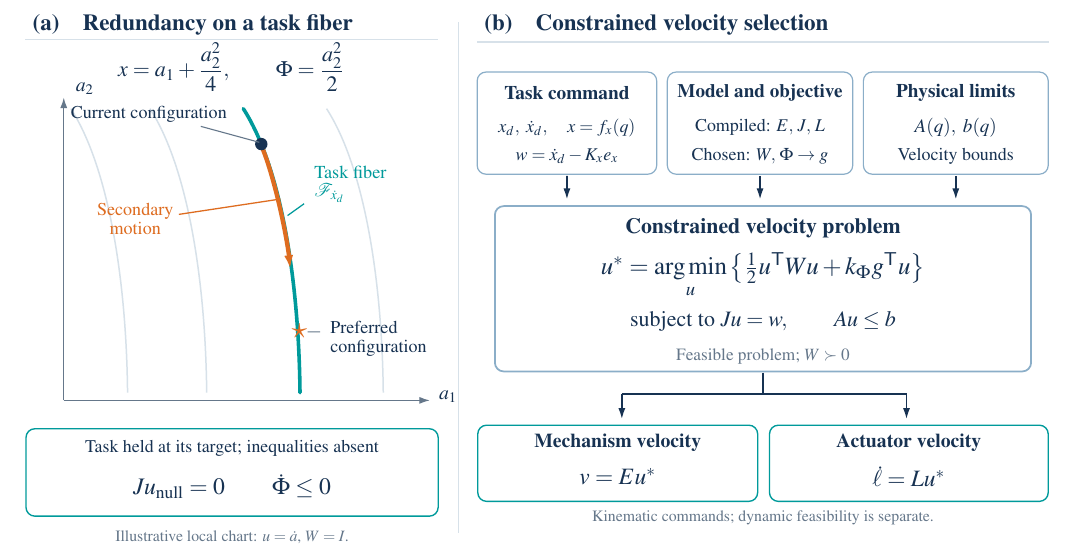}
  \caption{Constrained velocity selection with the compiled interfaces. (a) On a task fiber, the secondary objective $\Phi$ generates self-motion toward a preferred configuration while the task is held ($\dot\Phi\le0$); illustrative local chart with $u=\dot a$ and $W=I$. (b) Task commands, the compiled maps $E$, $J$, and $L$, the chosen metric $W$ and objective differential $g$, and physical velocity bounds define problem~\eqref{eq:qp}; its solution gives the mechanism velocity $v=Eu^*$ and actuator velocity $\dot\ell=Lu^*$.}
    \label{fig:equivalence}
\end{figure*}

\subsection{Three Meanings of Reconfiguration}
\label{subsec:reconfig}

\emph{Model reconfiguration} changes dimensions, attachments, or joint/module data between runs. It is handled by recompiling the affected paths, closures, transmissions, and inertia mappings from the edited physical records; the recompiled model describes the changed mechanism.

\emph{Computational reconfiguration} changes a tree, cycle basis,
or valid local velocity partition while preserving the physical
mechanism. On an overlap, state and speed must be transferred
through their actual coordinate transformations, including
$u=F\widetilde u$. A new partition is accepted only after complete
closure and rank conditions are re-established. The proposition
applies to a valid switch.

\emph{Physical reconfiguration during motion} adds or removes actual connections. Latching, grasping, impacts, and changing contact modes alter the admissible velocities. RoboCompiler compiles each support mode as a fixed-topology closure and schedules switches between compiled modes, as in the Go2 reference with 110 support transitions (Section~\ref{subsec:exp_tasks}); in the reported executions, impacts and contact transitions are resolved by the contact models of the executing backends. Nonholonomic rolling constraints act at the velocity level and can enter through a separate motion layer.

\section{Experimental Evaluation}
\label{sec:experiments}

\subsection{Robots, Environments, and Evaluation Protocol}
We evaluate RoboCompiler on five simulated robots
(Table~\ref{tab:exp_robots} and Fig. \ref{robots}). The industrial excavator, Kangaroo,
and six-UPS Stewart platform contain physical closed chains.
Go2 and Panda have physical trees augmented by foot-task and
support constraints, or by a finger coupling and tool target.
The excavator and Stewart compilers select trees and cuts from physical records; Kangaroo uses declared cuts.
The systems exercise moving actuator bodies, coordinate
partitions, floating bases, and task-dependent constraints.

MuJoCo and Isaac Sim/PhysX integrate native joints and loop
constraints, together with unilateral contact where applicable.
Pinocchio supplies independent rigid-body and constrained-dynamics
calculations; its task executions use separately implemented
contact or external-load models. Within each matched experiment,
the physical parameters, references, and actuator definitions
are held fixed. The task
studies use MuJoCo~3.3.7, Pinocchio~3.8.0, and Isaac Sim~6.1;
the separate solver comparison in Table~\ref{tab:rc_dynamics}
also tests Pinocchio~4.1 and MuJoCo~3.14.0. The Kangaroo, Panda, and Go2 models are derived from public MuJoCo models~\cite{kangaroo_sim2sim,menagerie}.

\begin{table*}[!t]
\centering\small
\setlength{\tabcolsep}{4pt}
\caption{Evaluation systems and nominal native Isaac Sim/PhysX
executions. The last column reports task-specific metrics.}
\label{tab:exp_robots}
\begin{tabular}{@{}lp{0.23\textwidth}p{0.27\textwidth}p{0.24\textwidth}@{}}
\toprule
Robot & Mechanism or imposed constraints & Task & Native PhysX result \\
\midrule
Industrial excavator \cite{pakkila2017modeling} & Nine revolute cuts; seven independent arm speeds & Excavation and settled soil transfer & 92.153 kg delivered; 11.46-$\mu$m peak arm gap \\
Unitree Go2 & Tree; four foot targets and two- to four-foot support & 26-s rail, turn, gate, and docking course & 6.926-mm base-position RMS ($t\geq2$ s) \\
Franka Panda & Tree; equal finger displacements and six-dimensional tool target & 22-s grasp, inspection, and keyed-socket placement & 0.133-mm final planar placement error \\
Kangaroo & Sixteen point and eight universal cuts; six base and twelve motor speeds & 10-s landing, crouch, and push recovery & 120.078-mm crouch; 0.629-$\mu$m peak point gap \\
Six-UPS Stewart & Six limbs with spherical closure; six independent lengths & 22-s six-axis inspection with platform wrenches & 0.717-mm platform-position RMS; 6.454-$\mu$m peak gap \\
\bottomrule
\end{tabular}
\end{table*}

\begin{figure*}[!t]
\centering
\includegraphics[width=0.75\textwidth]{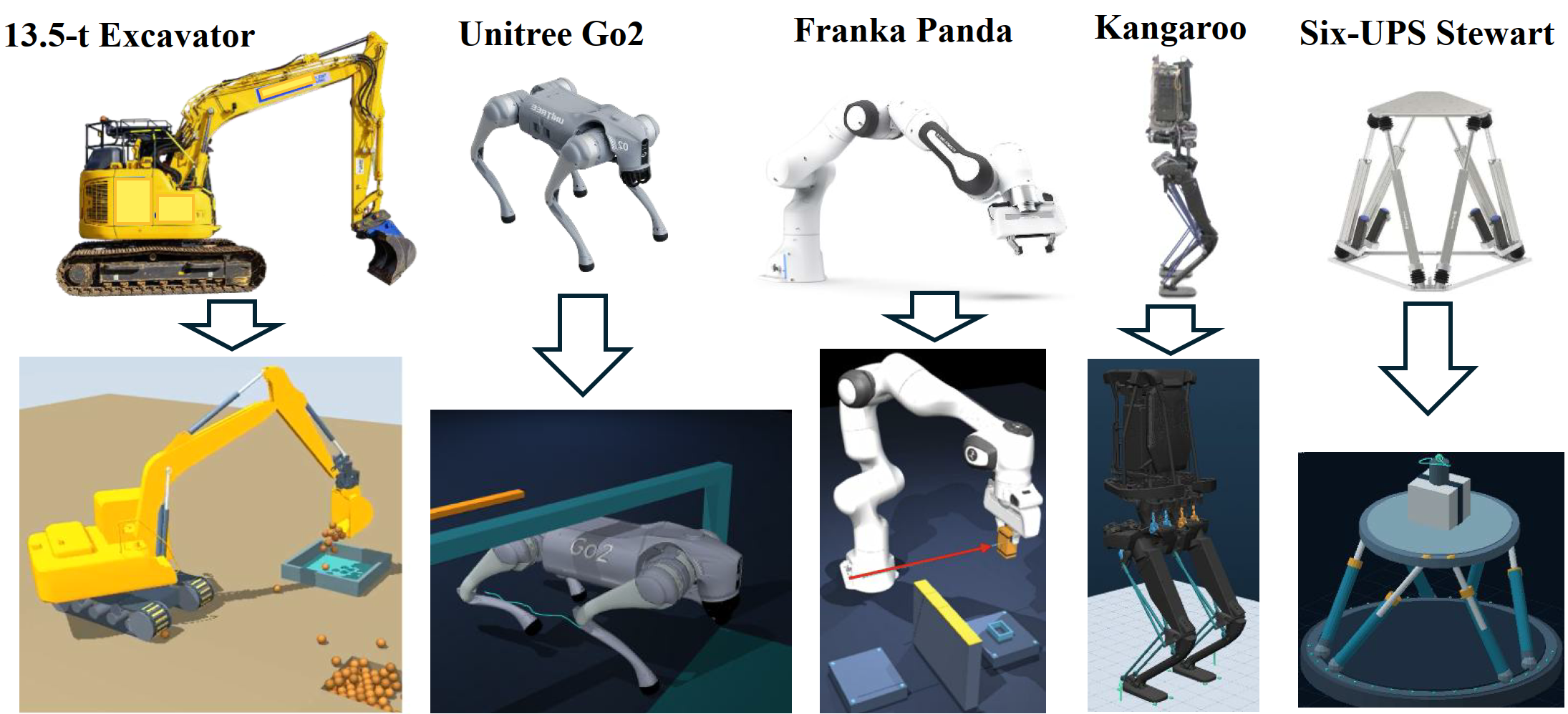}
\caption{The five evaluation platforms (top) and the simulation models used in all experiments (bottom): the excavator, Unitree Go2, Franka Panda, Kangaroo, and six-UPS Stewart platform.}
\label{robots}
\end{figure*}

The evaluation tests mechanical consistency, computational
savings, use of the actuation interface, and reuse in independent
plants. Statewise
ideal-support checks are separated from executions with changing
native contact. We likewise distinguish physical endpoint gaps
from augmented closure residuals. 

\subsection{Representation and Mechanical Consistency}
\label{subsec:exp_mechanics}

\textit{Compilation under changed data and representations.}
For Stewart, twelve variants covering geometry, payload, inertia,
body-frame expression, and record or actuator order are compiled
under all six spherical tree-joint choices. All 72 models pass
five feasible-configuration checks, giving 360 witnesses.
The generated physical tree has 21 rather than 24 coordinates,
while retaining six independent actuator lengths (Listing~\ref{lst:stewart}). Between paired
representations of the same mechanism, the largest normalized
reduced-inertia discrepancy is $3.77\times10^{-15}$ and the
largest task-map entry difference is $2.89\times10^{-15}$;
the maximum physical closure gap is $5.72\times10^{-16}$ m.
At perturbed off-manifold states, generated differentials agree
with centered directional differences within
$3.01\times10^{-10}$, and their observed nonzero entries remain
inside the generated sparsity pattern.

\begin{lstlisting}[float=tp, label={lst:stewart},
  caption={Abridged Stewart input and compiler calls. The six
  platform joints are ordinary spherical records. For each choice,
  the compiler places that joint in the tree as three chart
  coordinates and closes the other five as chords; $E$ is the
  lift in~\eqref{eq:lift}.}]
physical = {  # Stewart records, abridged
  "bodies": [...],  # 21, with mass and inertia
  "joints": [       # 25 undirected edges
    {"id": "leg_0_length", "type": "prismatic",
     "body_a": "leg_0_barrel",
     "body_b": "leg_0_rod", ...},
    {"type": "spherical", "body_a": "leg_0_rod",
     "body_b": "platform", ...},  # six such
    ...],
  "actuators": [...]}  # six leg-force ports

for s in spherical_ids:  # six tree-joint choices
    comp = compile_graph(
        physical, preferred_spherical_tree=s)
    # tree: 18 scalar + 3 chart coordinates;
    # 5 spherical chords: 30 residual rows
    graph = comp.graph(q0)  # q0: closed seed
    E, info = PACDM(graph).mapping(graph.lift(q0))
\end{lstlisting}

The excavator compiler is checked on 120 closed configurations:
twelve physical-data variants, two independent-coordinate
partitions, and five configurations per pair. Both joint and
cylinder-stroke partitions retain seven independent speeds.
Across 48 dynamics states, their projected accelerations agree
with independently assembled KKT and native Pinocchio solutions
to $5.7\times10^{-13}$ in the force-equivalent metric defined
below; the two partitions agree to $1.94\times10^{-15}$.
For Go2, twelve physical-data variants generate 144 task/support
plans and 720 passing configuration checks. These checks compare valid representations of each edited
mechanism, exercising recompilation from physical records.

\textit{Independent constrained dynamics.}
Let $a=\dot v_T$ denote generalized acceleration in the
common tree-velocity convention. For the comparison in
Table~\ref{tab:rc_dynamics}, define $\bar a=a$ for Stewart
and the excavator. For Kangaroo, $\bar a$ stacks the pelvis
classical linear and angular accelerations in world axes
and the 76 joint accelerations, avoiding dependence on the
floating-base parameterization. Against the corresponding
reference quantities, we report
\begin{equation}
\begin{aligned}
\delta_a&=
\frac{\|\bar a-\bar a_{\rm ref}\|_\infty}
{\max(1,\|\bar a_{\rm ref}\|_\infty)},\\
\delta_f&=
\frac{\|M_T(a-a_{\rm ref})\|_\infty}
{\max(1,\|\tau_T\|_\infty,\|h_T\|_\infty)}.
\end{aligned}
\label{eq:exp_errors}
\end{equation}
The force metric uses generalized accelerations and the
reference inertia and bias terms. Coordinates and SI
conventions are matched within each model. With physical
constraint matrix $J_{c,T}$ and drift $\gamma_{c,T}$,
the acceleration-level residual is
$r_c=\|J_{c,T}a+\gamma_{c,T}\|_\infty$.
Its components retain the units of their constraint rows, and all three metrics are compared within each robot.

\textit{Sampling and reference.}
Table~\ref{tab:rc_dynamics} uses sets~A and~B.
Set~A contains 1101 recorded Stewart reference knots and
48 recorded dynamics witnesses each for the excavator
and Kangaroo. Set~B adds 300, 252, and 152 states,
respectively, with configurations sampled from the same
reference motions and randomized independent velocities,
actuator efforts, and external wrenches. Independent
velocities are mapped through $v_T=E_T(q)u$.

At each state, Pinocchio~4.1 evaluates $M_T$, $h_T$,
$J_{c,T}$, and $\gamma_{c,T}$ in double precision.
The reference uses a 50-significant-digit SVD null-space
solve, retaining singular values above
$10^{-9}\sigma_1$ and treating redundant rows in the
least-squares sense at the selected numerical rank.
Returned accelerations are rounded to double precision.
Extended precision reduces solver round-off in the reference.

Every method receives the same physical configuration,
generalized velocity, applied efforts, and support mode,
converted to its native coordinate conventions.
Table~\ref{tab:rc_sampling} specifies the recorded and
randomized states. Configurations are sampled from the
1101-knot Stewart reference (0--22~s at 20-ms intervals),
the 4201-sample excavator route, and the 401-sample
Kangaroo reference motion. Randomized states use a NumPy
generator initialized with seed 20260927 for each robot, with a fixed draw order.

Independent velocities are mapped through $v_T=E_T(q)u$.
The independent variables are the six actuator lengths
for Stewart, the seven inputs of the excavator's joint-space
partition, and the coordinates selected by Kangaroo's
support plan. Kangaroo has 18, 12, or 6 independent speeds
in floating, single-support, or double-support mode,
respectively. Ideal sole welds are anchored at the current
sole frames. External wrenches are expressed in world axes
and applied at the platform origin, bucket frame, or pelvis
origin for Stewart, the excavator, or Kangaroo, respectively.

\begin{table*}[t]
\centering
\footnotesize
\setlength{\tabcolsep}{4pt}
\caption{Sampling protocol for the dynamics comparison.
$\mathcal U(\pm a)$ denotes independent uniform draws
in $[-a,a]$ for each component. Set~A contains recorded
states; set~B contains randomized velocities, efforts,
and wrenches at configurations sampled from the reference
motions.}
\label{tab:rc_sampling}
\begin{tabular}{@{}
>{\raggedright\arraybackslash}p{0.10\textwidth}
>{\raggedright\arraybackslash}p{0.25\textwidth}
>{\raggedright\arraybackslash}p{0.59\textwidth}@{}}
\toprule
Robot & Set A: recorded states & Set B: randomized states \\
\midrule
Stewart
&
All 1101 reference knots, with recorded actuator
velocities and feedforward forces; no external wrench.
&
300 states. Reference knot sampled uniformly;
actuator velocities $\mathcal U(\pm0.25)$~m/s;
actuator forces $\mathcal U(\pm400)$~N;
external force components $\mathcal U(\pm120)$~N
and moment components $\mathcal U(\pm12)$~N\,m.
\\
\addlinespace
Excavator
&
48 dynamics witnesses at route indices evenly spaced
from 20 to 4181.
&
252 states. Route index sampled uniformly from
20--4180; independent velocities
$\mathcal U(\pm0.3)$~rad/s.
Each effort is sampled as $\mathcal U(\pm1)$ times
its scale: 5, 5, 30, 150, 150, and 500~kN for the
six cylinders, and 1 and 25~kN\,m for the rotary drives.
External force components are $\mathcal U(\pm5)$~kN
and moment components are $\mathcal U(\pm2)$~kN\,m.
\\
\addlinespace
Kangaroo
&
48 dynamics witnesses at motion indices evenly spaced
from 8 to 392, cycling through floating, left-sole,
right-sole, and double-support modes.
&
152 states, with 38 per support mode.
Motion index sampled uniformly from 8--392;
independent velocities $\mathcal U(\pm0.12)$ in
their respective m/s or rad/s units;
motor efforts $\mathcal U(\pm200)$~N.
External force components use ranges
$\pm15$, $\pm15$, and $\pm10$~N;
moment components use $\mathcal U(\pm2)$~N\,m.
\\
\bottomrule
\end{tabular}
\end{table*}

The Stewart velocity interval covers the recorded peak
actuator speed of approximately 0.23~m/s, while the
sampled forces remain within the 900-N actuator limit.
Together, sets~A and~B contain 1401 Stewart, 300 excavator,
and 200 Kangaroo states.

\textit{Reference construction and numerical rank.}
Let
$J_{c,T}=U\Sigma V^\top$,
and let $U_r$, $V_r$, and $\Sigma_r$ contain the retained
singular vectors and values. With $Z$ containing the
remaining right singular vectors, the reference is
computed as
\begin{equation}
\begin{aligned}
a_p &=-V_r\Sigma_r^{-1}U_r^\top\gamma_{c,T},\\
(Z^\top M_TZ)y
    &=Z^\top(\tau_T-h_T-M_Ta_p),\\
a_{\rm ref}&=a_p+Zy.
\end{aligned}
\label{eq:rc_reference}
\end{equation}
The particular acceleration $a_p$ uses a truncated-SVD
least-squares treatment of the constraint equations.
At the selected numerical rank, positive definiteness
of $Z^\top M_TZ$ gives a unique projected solution.

The retained rank is 18 of 18 rows for Stewart and
16 of 54 rows for the excavator. For Kangaroo, it is
64 of 80 rows in floating motion, 70 of 86 rows with
one welded sole, and 76 of 92 rows with both soles
welded. These ranks are unchanged across the sampled
states and agree with the double-precision rank checks.

\textit{Reference sensitivity.}
To distinguish solver sensitivity from the pelvis-based
Kangaroo comparison in Table~\ref{tab:rc_dynamics}, define
the normalized tree-coordinate acceleration discrepancy
\begin{equation}
\delta_{\rm tree}
=
\frac{\|a-a_{\rm ref}\|_\infty}
{\max(1,\|a_{\rm ref}\|_\infty)}.
\label{eq:rc_tree_error}
\end{equation}
The largest differences between double- and
extended-precision solves are $1.6\times10^{-13}$,
$4.4\times10^{-9}$, and $1.2\times10^{-11}$ in
$\delta_{\rm tree}$ for Stewart, the excavator,
and Kangaroo, respectively. These values measure
solver sensitivity with the model inputs held fixed.

For the excavator, alternative well-conditioned
independent subsets of its redundant constraint rows
change the reference by up to $6.5\times10^{-10}$ in
$\delta_a$, with a median of $4.1\times10^{-11}$,
and by up to $2.3\times10^{-12}$ in $\delta_f$.
These observed spreads provide an empirical estimate
of reference sensitivity, rather than a rigorous error
bound. Differences below this estimated floor indicate
agreement with the chosen reference without resolving
an absolute accuracy ordering.

RoboCompiler's maximum $\delta_a$ is $1.6\times10^{-8}$,
$5.1\times10^{-7}$, and $4.4\times10^{-12}$ for Stewart,
the excavator, and Kangaroo. These are smaller than the tested
MuJoCo~3.14.0 maxima by approximately $4.2\times10^5$,
$4.7\times10^4$, and $6.1\times10^{10}$. Because MuJoCo
uses regularized equalities, the ratios measure agreement with
the hard-constraint reference under the reported settings.
Against Pinocchio's \texttt{lcaba} with $\mu=10^{-6}$ (with the Schur recipe on Kangaroo), RoboCompiler attains lower median $\delta_a$ and maximum $\delta_f$ on all three mechanisms and lower maximum $\delta_a$ on the excavator and Kangaroo; on Kangaroo, it also improves on \texttt{lcaba} with $\mu=10^{-2}$ in every accuracy metric. Pinocchio's \texttt{constraintDynamics}, with the stated Kangaroo Schur recipe, attains the smallest discrepancies overall.

For the excavator, RoboCompiler's maximum $\delta_f=1.1\times10^{-12}$
lies below the estimated $2.3\times10^{-12}$ reference floor arising
from alternate treatments of redundant rows, as do the $\delta_f$
entries of \texttt{lcaba} with $\mu=10^{-2}$ and
\texttt{constraintDynamics}; at the reference resolution, these three
solvers are indistinguishable in $\delta_f$. Call times compare
RoboCompiler's Python reference implementation with compiled-library
calls on precomputed arrays (except for the complete Kangaroo Schur
recipe) and therefore reflect implementation-specific costs with
different scopes; the computational benefit of compilation is measured
within a common implementation in Section~\ref{subsec:exp_computation}.

\begin{table*}[!t]
\centering
\caption{Constrained forward dynamics at identical feasible states.
Errors use a 50-digit hard-constraint KKT reference and
Eq.~\eqref{eq:exp_errors}; $r_c=\|J_{c,T}a+\gamma_{c,T}\|_\infty$
has components in m/s$^2$, rad/s$^2$, or s$^{-2}$ for
universal-axis rows. Time $t$, the median of per-state mean call times,
is measured separately
on 100 Stewart/excavator or 40 Kangaroo states, on one
2.80-GHz Xeon core with one BLAS thread. RoboCompiler is a Python
reference implementation. Pinocchio timings cover the native
solver call; MuJoCo timings include native-array assignment
and \texttt{mj\_forward}; Kangaroo timings cover the complete
Schur recipe. State conversion and model construction are
excluded from the native timing scopes. $^{\dagger}$Pinocchio~4.1 lacks native
universal-loop rows; eight rows are added by a Schur recipe
requiring ten native solver calls per acceleration.
$^{\ddagger}$Below the estimated excavator reference floor from
alternate well-conditioned independent subsets of redundant rows:
$\delta_a$ max $6.5\times10^{-10}$, median $4.1\times10^{-11}$;
$\delta_f$ max $2.3\times10^{-12}$. Such entries agree with the reference to within its resolution and are mutually indistinguishable.}
\label{tab:rc_dynamics}
\small
\setlength{\tabcolsep}{2pt}
\begin{tabular}{@{}lccccr@{}}
\toprule
Method & $\delta_a$ max & $\delta_a$ median & $\delta_f$ max & $r_c$ max & $t$ \\
\midrule
\multicolumn{6}{@{}l}{\textit{Stewart platform: 24 tree coordinates, 6 point cuts (18 rows, rank 18); 1,401 states}} \\
RoboCompiler (authors' code) & $1.6{\times}10^{-8}$ & $2.5{\times}10^{-10}$ & $4.9{\times}10^{-10}$ & $2.4{\times}10^{-7}$ & $22.3\,\mathrm{ms}$ \\
Pinocchio 4.1 \texttt{lcaba}, $\mu{=}10^{-6}$ & $9.3{\times}10^{-9}$ & $2.1{\times}10^{-9}$ & $1.5{\times}10^{-9}$ & $1.1{\times}10^{-13}$ & $19\,\mu\mathrm{s}$ \\
Pinocchio 4.1 \texttt{lcaba}, $\mu{=}10^{-2}$ & $9.8{\times}10^{-13}$ & $5.5{\times}10^{-13}$ & $1.4{\times}10^{-13}$ & $8.0{\times}10^{-13}$ & $28\,\mu\mathrm{s}$ \\
Pinocchio 4.1 \texttt{constraintDynamics} & $7.6{\times}10^{-14}$ & $3.3{\times}10^{-14}$ & $3.1{\times}10^{-15}$ & $6.4{\times}10^{-13}$ & $16\,\mu\mathrm{s}$ \\
MuJoCo 3.3.7 equalities & $2.6{\times}10^{-1}$ & $4.0{\times}10^{-3}$ & $3.0{\times}10^{-3}$ & $6.6{\times}10^{-1}$ & $22\,\mu\mathrm{s}$ \\
MuJoCo 3.14.0 equalities & $6.7{\times}10^{-3}$ & $4.1{\times}10^{-3}$ & $1.5{\times}10^{-4}$ & $3.8{\times}10^{-2}$ & $20\,\mu\mathrm{s}$ \\
\addlinespace
\multicolumn{6}{@{}l}{\textit{Excavator arm: 23 tree coordinates, 9 revolute cuts as 18 point pairs (54 rows, rank 16); 300 states}} \\
RoboCompiler (authors' code) & $5.1{\times}10^{-7}$ & $8.5{\times}10^{-11}$ & $1.1{\times}10^{-12}{}^{\ddagger}$ & $3.4{\times}10^{-11}$ & $12.6\,\mathrm{ms}$ \\
Pinocchio 4.1 \texttt{lcaba}, $\mu{=}10^{-6}$ & $6.9{\times}10^{-6}$ & $6.9{\times}10^{-7}$ & $2.2{\times}10^{-10}$ & $3.1{\times}10^{-12}$ & $35\,\mu\mathrm{s}$ \\
Pinocchio 4.1 \texttt{lcaba}, $\mu{=}10^{-2}$ & $4.6{\times}10^{-9}$ & $1.6{\times}10^{-10}$ & $2.8{\times}10^{-14}{}^{\ddagger}$ & $3.1{\times}10^{-12}$ & $109\,\mu\mathrm{s}$ \\
Pinocchio 4.1 \texttt{constraintDynamics} & $4.2{\times}10^{-12}{}^{\ddagger}$ & $4.0{\times}10^{-13}{}^{\ddagger}$ & $4.3{\times}10^{-15}{}^{\ddagger}$ & $7.2{\times}10^{-12}$ & $73\,\mu\mathrm{s}$ \\
MuJoCo 3.3.7 equalities & $2.4{\times}10^{-2}$ & $2.3{\times}10^{-3}$ & $5.3{\times}10^{-5}$ & $3.2{\times}10^{-1}$ & $33\,\mu\mathrm{s}$ \\
MuJoCo 3.14.0 equalities & $2.4{\times}10^{-2}$ & $2.3{\times}10^{-3}$ & $2.9{\times}10^{-5}$ & $2.7{\times}10^{-1}$ & $32\,\mu\mathrm{s}$ \\
\addlinespace
\multicolumn{6}{@{}l}{\textit{Kangaroo: 82 generalized velocities incl. base, 16 point + 8 universal cuts (80 rows, rank 64) + 6 per welded sole; 200 states}} \\
RoboCompiler (authors' code) & $4.4{\times}10^{-12}$ & $1.4{\times}10^{-13}$ & $2.8{\times}10^{-13}$ & $8.0{\times}10^{-10}$ & $52.6\,\mathrm{ms}$ \\
\texttt{lcaba} + Schur$^{\dagger}$, $\mu{=}10^{-6}$ & $3.3{\times}10^{-4}$ & $5.1{\times}10^{-5}$ & $3.4{\times}10^{-5}$ & $5.1{\times}10^{-3}$ & $3.0\,\mathrm{ms}$ \\
\texttt{lcaba} + Schur$^{\dagger}$, $\mu{=}10^{-2}$ & $5.3{\times}10^{-9}$ & $1.2{\times}10^{-9}$ & $9.6{\times}10^{-12}$ & $2.1{\times}10^{-5}$ & $3.7\,\mathrm{ms}$ \\
\texttt{constraintDynamics} + Schur$^{\dagger}$ & $5.3{\times}10^{-13}$ & $3.9{\times}10^{-15}$ & $2.4{\times}10^{-15}$ & $1.4{\times}10^{-11}$ & $5.6\,\mathrm{ms}$ \\
MuJoCo 3.3.7 equalities & $2.7{\times}10^{-1}$ & $2.4{\times}10^{-3}$ & $2.9{\times}10^{-3}$ & $1.5{\times}10^{1}$ & $125\,\mu\mathrm{s}$ \\
MuJoCo 3.14.0 equalities & $2.7{\times}10^{-1}$ & $3.2{\times}10^{-3}$ & $2.9{\times}10^{-3}$ & $2.3{\times}10^{1}$ & $125\,\mu\mathrm{s}$ \\
\bottomrule
\end{tabular}
\end{table*}

\textit{Floating bases, support modes, and port power.}
Separate double-precision checks evaluate 48 Go2 states spanning
all eleven tested two-, three-, and four-foot ideal-support sets.
Their constraint ranks are 6, 9, and 12, leaving 12, 9, and 6
independent speeds. Maximum normalized acceleration discrepancies
are $3.225\times10^{-13}$ against native Pinocchio and
$3.252\times10^{-13}$ against independently assembled KKT.
For Kangaroo, 48 states cover free motion, either sole welded,
and both soles welded. Ideal sole constraints increase the rank
from 64 to 70 or 76, leaving 18, 12, or 6 independent speeds;
the corresponding discrepancy maxima are
$1.86\times10^{-12}$ and $1.85\times10^{-12}$.
Kangaroo's maximum actuator virtual-power defect is
$1.43\times10^{-14}$ W. Panda's sixteen mechanics configurations
give a maximum port virtual-work defect of
$4.26\times10^{-14}$ W. These checks test the tangent, acceleration, and dual-effort interfaces at specified support modes; contact transitions are exercised in the native executions of Section~\ref{subsec:exp_tasks}.

\textit{Acceleration curvature.}
Figure~\ref{fig:exp_curvature} connects independent dynamics
checks to an omission study. For the 21-coordinate compiled
Stewart tree, 48 feasible states give a maximum normalized
acceleration discrepancy of $3.051\times10^{-13}$ against
native Pinocchio. At four times the prescribed actuator speed,
the maximum physical point-acceleration residual is
$8.876\times10^{-11}$ m/s$^2$ with curvature and
$1.010$ m/s$^2$ without it. Go2 gives corresponding residuals
of $8.676\times10^{-11}$ and $1.445$ m/s$^2$.
The curvature correction thus lowers the acceleration residual by about ten orders of magnitude, showing that acceleration compatibility requires it in addition to position closure and a valid tangent map. The curvature is evaluated by a centered directional difference of the analytic closure Jacobian (Listing~\ref{lst:jdot}).

\begin{lstlisting}[float=tp, label={lst:jdot},
  caption={Velocity-product and curvature terms. \texttt{C}
  returns the analytic closure Jacobian $C$ of~\eqref{eq:closure}
  in chart coordinates, where $v=\dot q$; no coordinate moves by
  more than $h_0$. With $S$ and $p$ from~\eqref{eq:partition}, the
  last lines form $c=\dot Eu$, whose tree rows enter
  the bias term of~\eqref{eq:dynamics}. Some workflows normalize
  by $\|\dot q\|_2$ or use $h_0=10^{-6}$, which keeps this bound.}]
def jdot_qdot(C, q, qdot, h0=1e-5):
    """Cdot(q, qdot) qdot from the analytic
    closure Jacobian C(q), by a centered
    directional difference along qdot."""
    h = h0 / max(1.0, np.max(np.abs(qdot)))
    dC = (C(q + h*qdot) - C(q - h*qdot)) / (2*h)
    return dC @ qdot  # second-order accurate

v = E @ u
cv = S @ jdot_qdot(C, q, v)  # selected rows
c = np.zeros_like(v)         # c = Edot u
c[p] = -solve(S @ C(q)[:, p], cv)
\end{lstlisting}

\begin{figure*}[!t]
\centering
\includegraphics[width=0.85\textwidth]{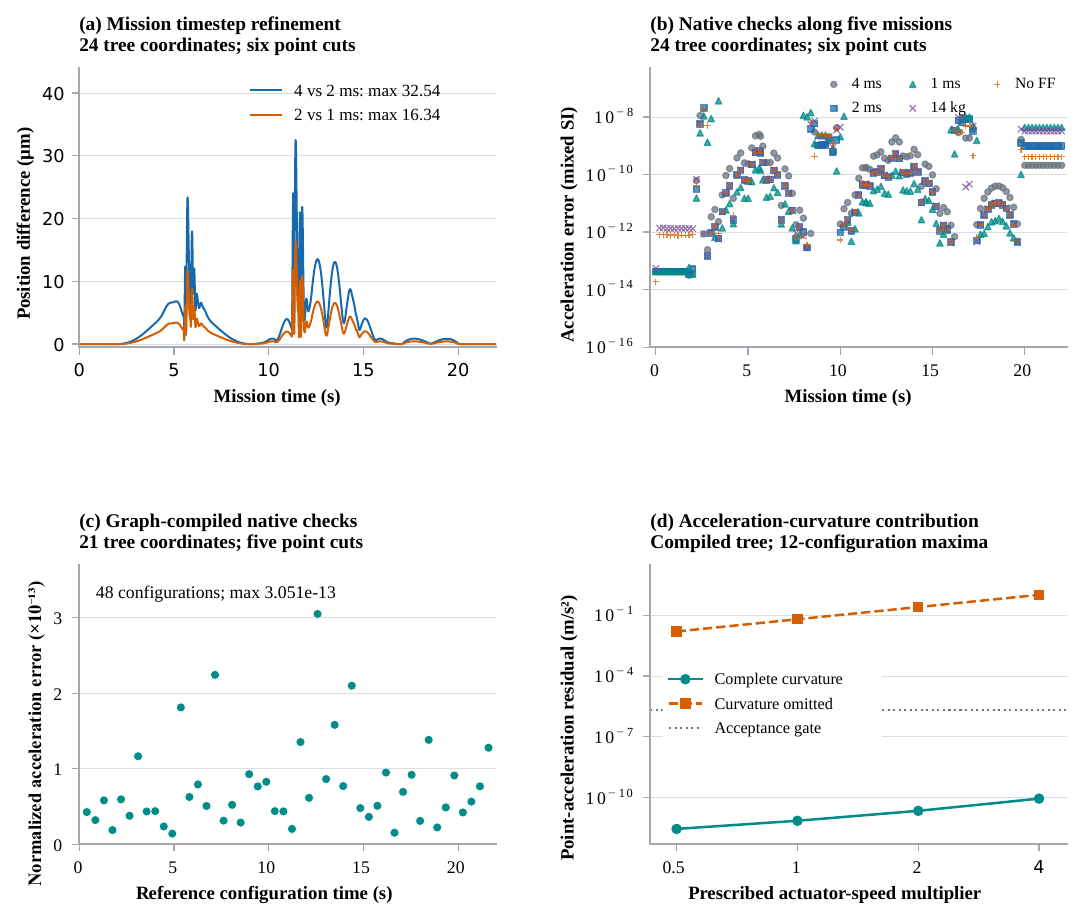}
\caption{Stewart dynamics verification. (a) Position differences under successive
integration-step refinements. (b) Native constrained-acceleration
checks at 113 states per mission case, in mixed SI coordinates.
(c) Normalized acceleration discrepancies at 48 configurations
of the compiled tree. (d) Physical point-acceleration residuals
at four prescribed speed multipliers, with and without curvature;
the dotted line denotes the recorded acceptance threshold.
Panels (a,b) use the 24-coordinate mission tree with six point
cuts; (c,d) use the 21-coordinate compiled tree with five cuts.}
\label{fig:exp_curvature}
\end{figure*}

\subsection{Generated Evaluation and Selective Updates}
\label{subsec:exp_computation}
The computational studies distinguish three effects: replacing
global path evaluation by generated cycle-local evaluation,
reusing modules whose complete inputs are unchanged, and supplying
analytic rather than numerical corrector derivatives.
Table~\ref{tab:exp_timing} summarizes matched comparisons.

For fixed excavator cuts, the generated residual and analytic
Jacobian agree with the original evaluator within
$7.44\times10^{-15}$. Dependency grouping reduces the largest
correction block from 25 to ten coordinates (Listing~\ref{lst:excavator}). Over a 42-s,
4201-sample reference, loop-free exclusion and modular updates
reduce summed correction-plus-output-map time by 25.3\% in
joint coordinates and 32.1\% with held cylinder strokes. The seven
independent speeds are preserved, and the reduced dynamics remain
physically coupled.

\begin{lstlisting}[float=tp, label={lst:excavator},
  caption={Abridged excavator input and compiler calls. No joint
  is marked as a cut: the compiler selects nine revolute cuts,
  including \texttt{q5}, checks each requested partition against
  the mobility, and groups the 25 dependent coordinates into
  modules, which are then corrected as in
  Listing~\ref{lst:kangaroo}.}]
physical = {  # excavator records, abridged
  "bodies": [...],  # 25, with mass and inertia
  "joints": [       # 34 edges; no cut is marked
    {"id": "q5", "type": "revolute", ...,
     "body_a": "body_54", "body_b": "body_3"},
    ...],
  "actuators": [...],  # 8 actuator ports
  "partition_requests": {
    "joint_space": ["q23", "q7", "q4", "q0",
                    "q1", "q21", "q22"],
    "cylinder_space": ["q23", "p3", "p5", "p2",
                       "p1", "q21", "q22"]},
  "seed": {...}, "branch_window": {...}}

comp = compile_graph(physical, "joint_space")
comp.plan["cut_joint_ids"]  # 9 revolute cuts
comp.plan["mobility"]       # 7
comp.plan["module_sizes"]   # [3, 3, 10, 3, 3, 3]
comp = compile_graph(physical, "cylinder_space")
comp.plan["module_sizes"]   # [6, 10, 6, 3]
\end{lstlisting}

Go2's local edits hold the prescribed base pose fixed, so changing
one foot target requires correction of only its three-coordinate
leg block (Listing~\ref{lst:go2}). One- and two-foot updates reduce median update
time, including output mapping, by 61.8\% and 28.1\%, respectively,
relative to generated global updates. For Kangaroo, common-ancestor pruning limits
closure paths to at most three joint steps. The table separately measures the global evaluator and
one-module update, including output-map construction. These
comparisons preserve the PACDM equations while changing evaluation
or scheduling.

\begin{lstlisting}[float=tp, label={lst:go2},
  caption={Abridged Go2 input and task-module calls. Feet are
  point sites, not closure joints. The compiler generates the
  floating-base chart, one three-coordinate module per foot, and
  the eleven two- to four-foot support plans; \texttt{step}
  re-solves only modules whose base or foot inputs changed.}]
physical = {  # Go2 records, abridged
  "floating_root": "base",
  "bodies": [...],  # 13, with mass and inertia
  "joints": [...],  # 12 revolute edges
  "actuators": [...],  # 12 motor ports
  "point_sites": [  # feet as task/contact sites
    {"id": "FL", "body": "FL_calf",
     "point_m": [-0.002, 0.0, -0.213], ...}, ...],
  "seed": {...}}

comp = compile_graph(physical)
# generated: base chart + 12 joints, and one
# three-coordinate module per foot site
plans = [comp.support_plan(s)
         for s in support_sets()]  # 11 modes
solver = ModularSolver(comp, q0)
q, E, info = solver.step(np.r_[q_b, p_feet])
# q_b: base pose (6); p_feet: foot targets (12)
\end{lstlisting}

\begin{table*}[!t]
\centering\footnotesize
\setlength{\tabcolsep}{5pt}
\caption{Matched computational ablations. Times are medians
except the final row, which averages two counterbalanced runs
per evaluator. Scheduled excavator times are median summed
correction-plus-map times over three traversals. Numerical
derivative pairs use matched correctors and stopping criteria.
Each row is a matched pair measured on a single host and workload. FD denotes finite differences.}
\label{tab:exp_timing}
\begin{tabular}{@{}llrrr@{}}
\toprule
Robot & Paired operation & Reference & Generated/local & Reduction \\
\midrule
Excavator & Residual + analytic Jacobian (ms) & 4.097 & 0.984 & 76.0\% \\
Excavator & Joint-reference correction + maps (s) & 24.117 & 18.026 & 25.3\% \\
Excavator & Held-cylinder correction + maps (s) & 24.493 & 16.628 & 32.1\% \\
Go2 & One-foot update (ms) & 3.467 & 1.326 & 61.8\% \\
Go2 & Colored FD / analytic correction (ms) & 9.657 & 4.311 & 55.4\% \\
Kangaroo & Residual + analytic Jacobian (ms) & 10.172 & 0.339 & 96.7\% \\
Kangaroo & One-module update (ms) & 8.963 & 1.979 & 77.9\% \\
Stewart & Original / compiled assembly + maps (ms) & 24.924 & 21.349 & 14.3\% \\
Stewart & Dense FD / analytic correction (ms) & 224.085 & 16.705 & 92.5\% \\
Stewart & Colored FD / analytic correction (ms) & 102.621 & 25.338 & 75.3\% \\
Kangaroo & Matched 10-s Pinocchio rollout (s) & 585.594 & 194.366 & 66.8\% \\
\bottomrule
\end{tabular}
\end{table*}

Stewart separates representation, prediction, and derivative
effects (Fig.~\ref{fig:stewart_compiler_benefits}).

\begin{figure}[h!]
\hspace*{-0.0cm} % Adjust the value as needed
\centering
\scalebox{1}{\includegraphics[trim={0cm 0.0cm 0.0cm 0cm},clip,width=\columnwidth]{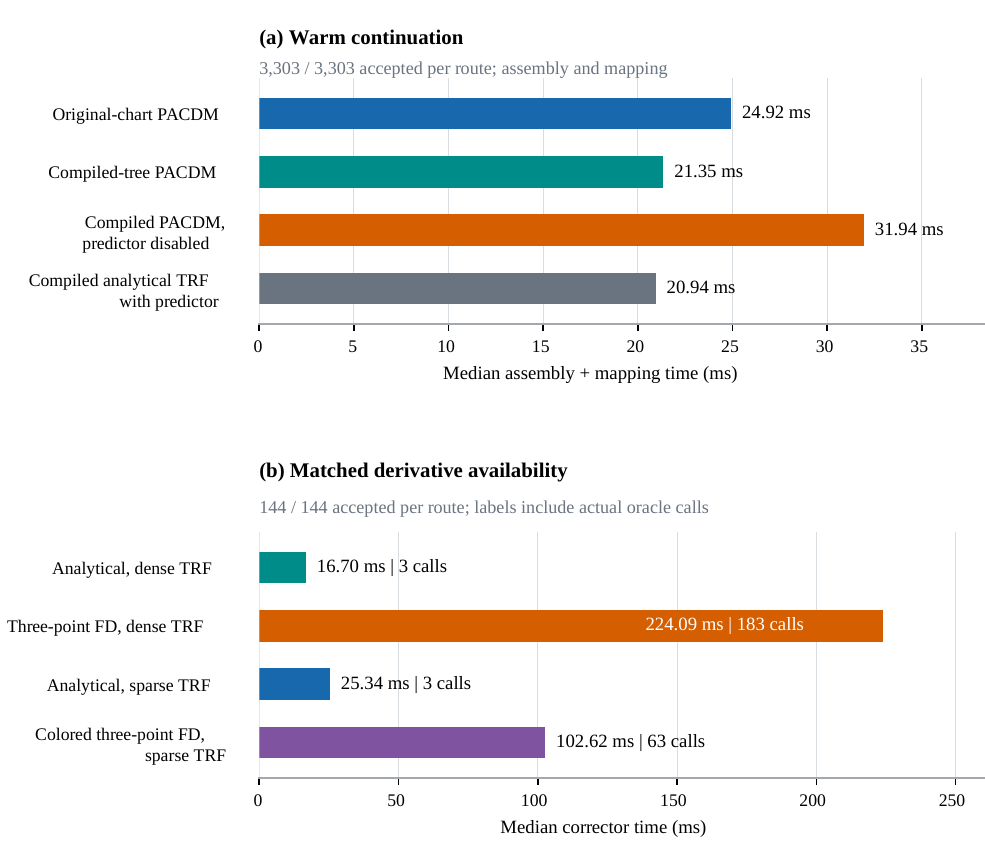}}
\caption{Computational evaluation from the recorded full-profile run;
bars show medians. (a) Assembly-plus-mapping time over three 1,101-knot traversals for
the original and compiled PACDM representations, the compiled
predictor-disabled route, and same-chart analytical TRF.
(b) Matched TRF corrector timings with analytical or three-point
finite-difference (FD) derivatives: dense modes share the exact dense
subsolver; sparse/colored modes share LSMR. All attempts are accepted.
Oracle counts include finite-difference perturbation evaluations.
The timing scopes in the two panels are evaluated separately.}
\label{fig:stewart_compiler_benefits}
\end{figure}

Over three traversals of
1101 knots, all four tested continuation methods accept all
3303 samples. Compiling the physical-tree representation reduces
median assembly-plus-mapping time by 14.3\%; tangent prediction
reduces it by 33.2\% relative to the 31.943-ms predictor-disabled
case. In separate matched corrector trials, all 144 attempts per
mode succeed. Analytic Jacobians reduce median fresh evaluations
from 183 to three against dense three-point differences and
from 63 to three against sparsity-colored differences. Dense
pairs share the exact dense TRF subsolver; sparse pairs share
LSMR. The benefit therefore persists when numerical differentiation
exploits the generated sparsity. Go2 similarly reduces fresh evaluations from
21 colored evaluations to three analytic evaluations.

The complete Kangaroo rollout shows that evaluator savings persist inside closed-loop integration
(Fig.~\ref{fig:exp_rollout}). Only the graph evaluator changes;
the PACDM solver, dynamics, contact algorithm, controller, and
1-ms integration step remain identical. Counterbalanced
original--generated--generated--original runs on one 2.10-GHz
Xeon host with one BLAS thread reduce mean elapsed time from
585.594 to 194.366 s. Both versions execute 10000 steps,
37870 fresh graph evaluations, and 111 correction iterations.
Maximum differences are $7.64\times10^{-14}$ m in motor
position, $1.93\times10^{-10}$ in reduced-velocity components,
and $1.20\times10^{-10}$ N\,s in net foot-impulse components.
Because only the graph evaluator changes, the 3.0-fold speed-up is attributable to compilation alone, and the differences above confirm that the computed response is preserved.

\begin{figure*}[!t]
\centering
\includegraphics[width=0.86\textwidth]{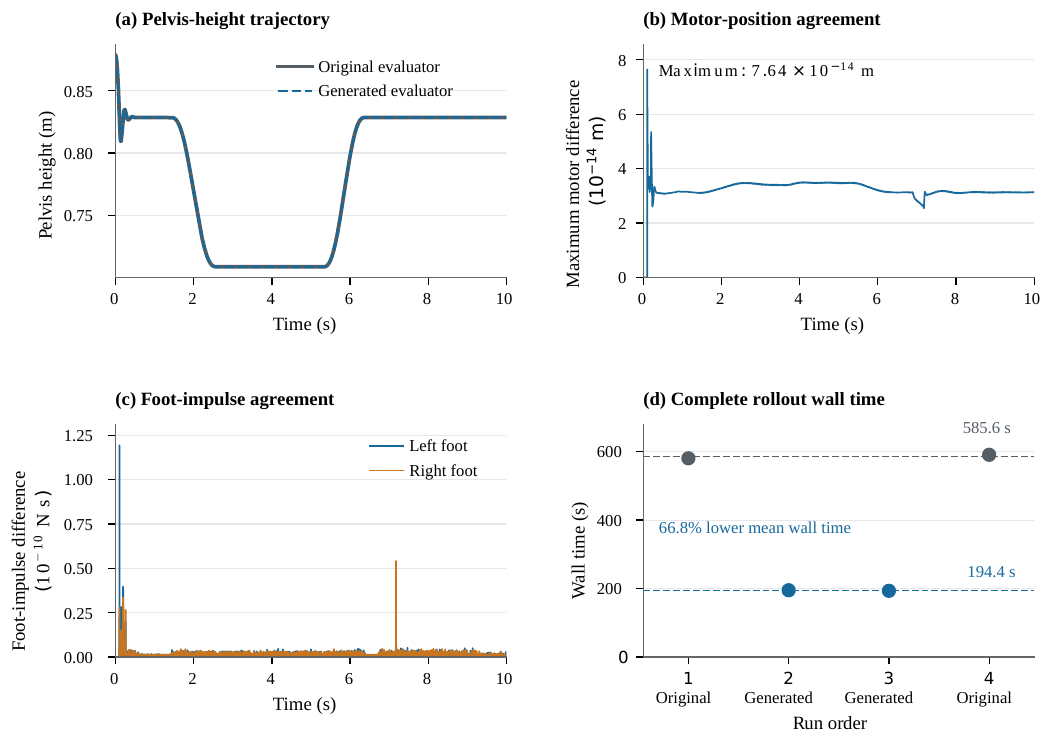}
\caption{Matched 10-s Kangaroo Pinocchio executions. (a) Pelvis
height. (b) Motor-position differences. (c) Net per-foot impulse
differences after summing the four contact-corner impulses.
(d) Elapsed times and evaluator means from four counterbalanced
runs. Histories use all 10001 samples; only the graph evaluator
changes between implementations.}
\label{fig:exp_rollout}
\end{figure*}

\subsection{Use of the Compiled Interfaces in Actuation}
\label{subsec:exp_actuation}
Matched feedforward ablations quantify the benefit of actuator commands derived from the compatible motion and dynamics interfaces.
In Kangaroo's 4-s fixed-pelvis experiment, both runs use the same
reference, proportional gain $4\times10^4$ N/m, derivative gain
600 N\,s/m, and 100-$\mu$s physics step.
Adding reduced-dynamics feedforward decreases aggregate
twelve-motor position RMS from 1.653 mm to
$0.647\,\mu$m. Every motor improves, and the peak actuator
force decreases from 250.70 to 239.84 N
(Fig.~\ref{fig:exp_feedforward}). This experiment isolates
feedforward use under a fixed base; floating-base contact
execution is evaluated separately below.

\begin{figure*}[!t]
\centering
\includegraphics[width=0.96\textwidth]{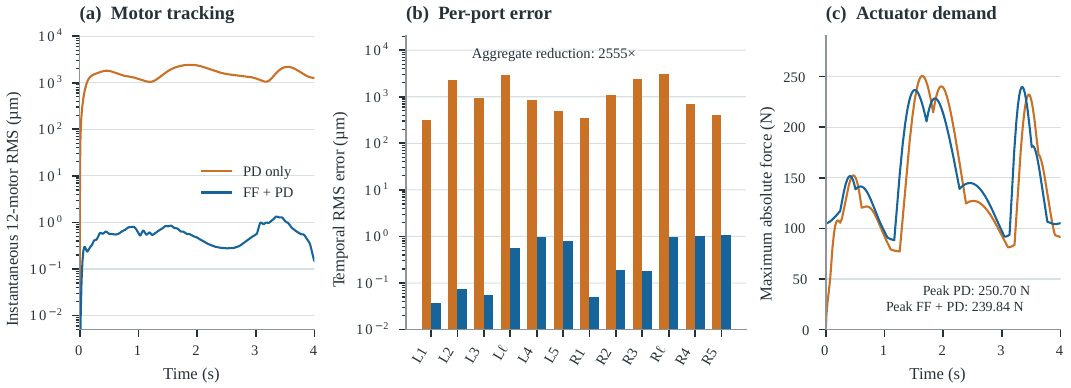}
\caption{Matched Kangaroo actuation experiment with a fixed
pelvis. (a) Instantaneous RMS motor-position error. (b) Temporal
RMS for each motor; $L\ell$ and $R\ell$ denote leg-length
ports. (c) Maximum absolute actuator force. The reference,
feedback gains, and physics step are identical; the changed
component is the compatible reduced-dynamics feedforward.}
\label{fig:exp_feedforward}
\end{figure*}

Feedforward derived from the same interface also reduces tracking error on Stewart and Panda.
For Stewart, feedforward reduces nominal platform-position RMS
from 2.375 to 0.717 mm in MuJoCo and direct PACDM--Pinocchio,
and from 2.374 to 0.717 mm in PhysX. Increasing the actual
payload from 8 to 14 kg while retaining the nominal feedforward
model gives 0.960, 0.959, and 0.959 mm, respectively.
All four PhysX cases, including timestep refinement, respect
the stroke and force limits, with peak force below 152 N
against the 900-N bound.

For Panda, the same-reference MuJoCo arm ablation reduces
tool-position RMS from 6.224 to 0.152 mm; the removed terms
include inverse dynamics, armature-acceleration compensation,
and damping compensation. A separate Pinocchio tool-wrench
execution gives 6.326 versus 0.150 mm. Across Kangaroo, Stewart, and Panda, these ablations show that the generated motion-to-effort interface directly supports model-based feedforward; computational savings are isolated in Section~\ref{subsec:exp_computation}.

\subsection{Native Task Execution and Perturbations}
\label{subsec:exp_tasks}

\textit{Excavation with physical material transfer.}
The excavator completes filled-bed and exposed-face MuJoCo
missions with 117.286 and 83.776 kg of identity-qualified
settled delivery. A particle contributes only after it has
been lifted in the bucket, released, and observed to settle
inside the receiving region. Empty-bed and disabled-bucket-contact
controls produce zero lifted or delivered mass. All 8856
PACDM updates across the two positive missions succeed without
reacquisition, and the native arm gaps remain below
$4.92\,\mu$m.

The native PhysX execution completes the eleven phases in
42.001 s, with peak arm and track gaps of $11.46\,\mu$m
and 0.423 mm, respectively
(Fig.~\ref{fig:exp_excavation}). Its 4202 controller evaluations
require no reacquisition. Eight fixed-base Pinocchio
cutting-and-payload cases also complete all phases, with maximum
acceleration-closure component $1.90\times10^{-11}$ m/s$^2$.
Those cases use prescribed smooth cutting and payload loads and validate force-driven mechanism execution.

\begin{figure*}[!t]
\centering
\includegraphics[width=0.80\textwidth]{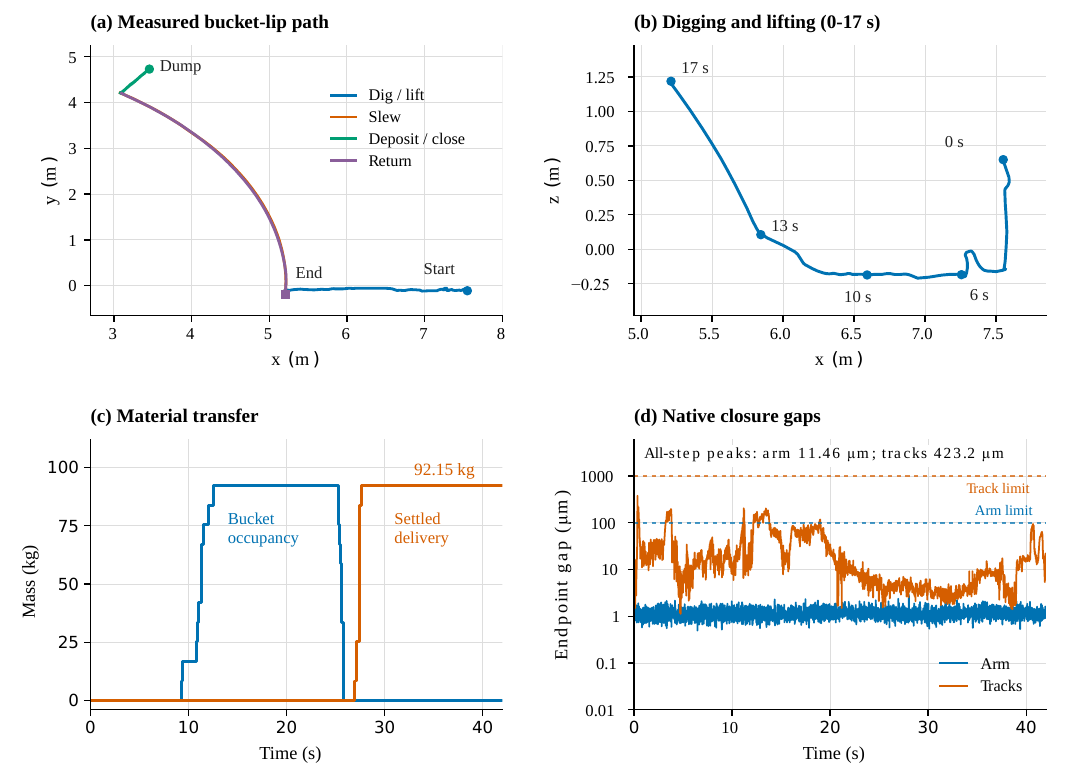}
\caption{Native PhysX excavation. (a) Complete bucket-lip plan-view
trajectory. (b) Vertical trajectory during excavation and
lifting (0--17 s).
(c) Material retained in the bucket and identity-qualified
settled delivery. (d) Physical arm and track closure gaps.
Native plant gaps are measured separately from the controller's
assembled PACDM closure residual.}
\label{fig:exp_excavation}
\end{figure*}

\textit{Floating-base course execution.}
Go2 completes its 26-s course in MuJoCo, the separately
integrated Pinocchio contact model, and PhysX. Nominal
base-position RMS errors for $t\geq2$ s are 10.60, 5.991,
and 6.926 mm, respectively. The 2601-sample PACDM reference
contains 110 scheduled support transitions and requires no
fallback. During native PhysX execution, measured unilateral
contacts produce two-, three-, and four-foot support.
Table~\ref{tab:exp_go2_perturbation} reports refinement,
friction, payload, and push cases. All cases complete the course, including the 25\%-stronger-push case, whose peak motor command reaches the authored limit. With an unmodeled 1.5-kg payload, Go2 still completes the course with 8.897-mm RMS and 6.655-mm final error.

\begin{table}[!t]
\centering\small
\setlength{\tabcolsep}{3pt}
\caption{Go2 native PhysX perturbations. RMS uses $t\geq2$ s;
final error is at 26 s. Motor utilization is the maximum
command normalized by the authored limit.}
\label{tab:exp_go2_perturbation}
\begin{tabular}{@{}lrrr@{}}
\toprule
Case & RMS (mm) & Final (mm) & Motor (\%) \\
\midrule
Nominal, 1-ms step & 6.926 & 0.179 & 68.43 \\
0.5-ms step & 6.736 & 0.474 & 64.22 \\
Friction 0.55 & 6.855 & 0.265 & 72.26 \\
1.5-kg payload & 8.897 & 6.655 & 59.88 \\
25\% stronger pushes & 7.237 & 0.222 & 100.00 \\
\bottomrule
\end{tabular}
\end{table}

\textit{Manipulation with coupled fingers.}
Panda completes the 22-s task with a 0.715-mm nominal final
planar error in MuJoCo and the PhysX result in
Table~\ref{tab:exp_robots}. Four PhysX cases include a halved step,
a 0.30-kg lower-friction cartridge, a 4-mm pickup offset,
and a 1-mm socket side clearance. All maintain two-finger
contact during the 8--16-s transfer and finish without arm
saturation, command clipping, or unexpected-contact steps.
The largest final planar error is 0.840 mm, and minimum
payload-corner clearance during obstacle transfer is 92.541 mm.
These results exercise the finger equality, tool reference,
and effort interface together with native grasp contact;
the independent Pinocchio experiment instead applies a
prescribed tool wrench.

\textit{Closed-chain landing and recovery.}
Kangaroo completes a 10-s sequence containing a 5-cm landing,
12-cm crouch, lateral and yaw motion, and a torso push.
Across five shared MuJoCo/Pinocchio conditions, including lower
friction, increased release height and push, timestep refinement,
and actuator response delay, the largest paired pelvis-position
difference is 1.576 mm. Native PhysX also completes the task,
achieving a 120.078-mm crouch within the individual force bounds.
Its nominal maximum point gap is $0.629\,\mu$m and the
dimensionless universal-axis residual is $3.05\times10^{-4}$.
Reporting both tests the translational and angular parts of the
universal-cut realization. MuJoCo and PhysX use a nominal
25-$\mu$s step, whereas the Pinocchio contact integration uses
1 ms; the agreement between MuJoCo and Pinocchio to within 1.576 mm therefore demonstrates model reuse across simulators whose integration steps differ by a factor of 40.

\begin{figure*}[h!]
\centering
\includegraphics[width=0.75\textwidth]{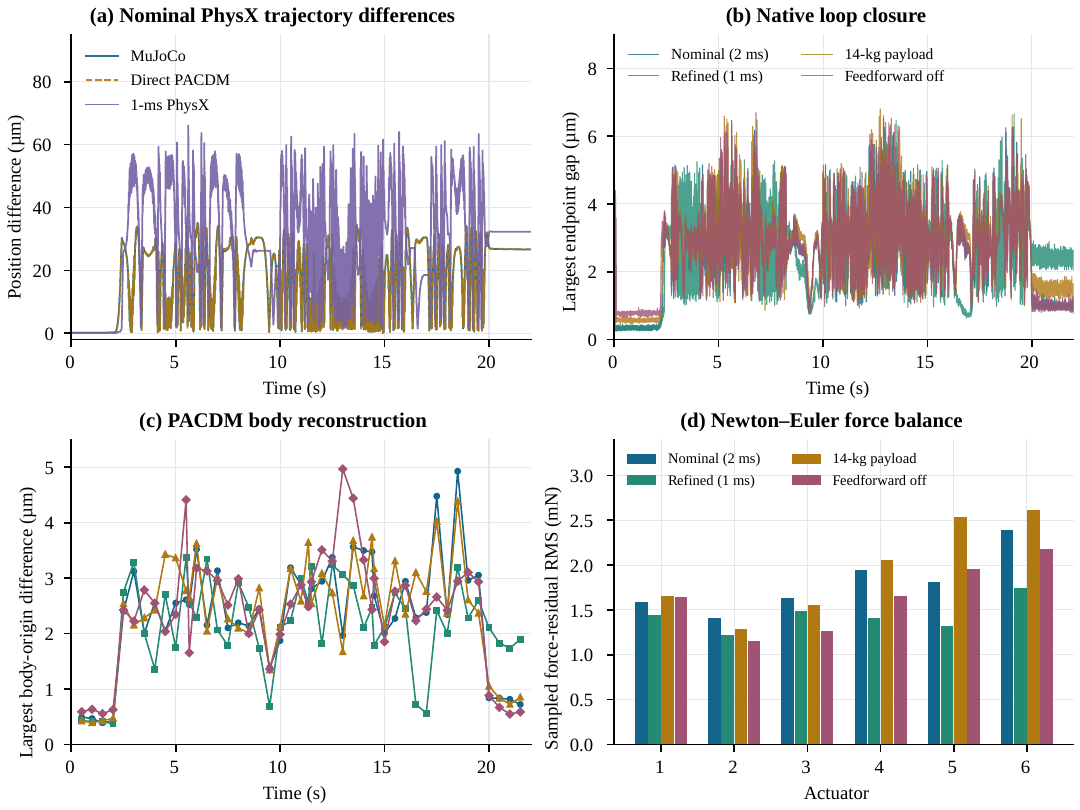}
\caption{Independent Stewart PhysX validation. (a) Position
differences from MuJoCo, direct PACDM--Pinocchio, and the
refined PhysX run at matched timestamps. (b) Maximum native
spherical-closure gaps. (c) Body-origin discrepancies following
fresh assembly at 46 measured states per case.
(d) Per-actuator RMS Newton--Euler force-balance residuals from
Eq.~\eqref{eq:exp_native_balance}. Panels (c,d) use measured
native states and freshly evaluated PACDM mappings.}
\label{fig:exp_native_transfer}
\end{figure*}

\subsection{Independent Native-State and Work Checks}
\label{subsec:exp_native}
Trajectory agreement is complemented by reconstruction and
mechanics checks at measured native states. For Go2, 55 PhysX
readback checks verify the transferred masses, centers of mass,
inertia tensors, and armatures. Twelve free-flight probes cover
six states at both 0.1- and 0.05-ms steps. Measured acceleration
agrees with the CMG/Pinocchio prediction to a maximum normalized
2-norm discrepancy of $1.25\times10^{-5}$; this uses
$\|a-a_{\rm ref}\|_2/\max(1,\|a_{\rm ref}\|_2)$,
distinct from Eq.~\eqref{eq:exp_errors}.
During the nominal course, native link poses and CMG forward
kinematics give foot positions within $0.783\,\mu$m.

Stewart provides a direct test of the closed-chain transfer
(Fig.~\ref{fig:exp_native_transfer}). Four PhysX missions
comprise nominal, refined-step, increased-payload, and
feedforward-disabled cases. At 46 prescribed timestamps per
case, measured actuator lengths and native configuration seeds
are used to assemble PACDM afresh. All 184 assemblies retain
full/passive augmented rank 36 and positive-definite reduced
inertia. Reconstructed body origins agree with native states
within $4.971\,\mu$m; maximum center-of-mass linear and
angular velocity discrepancies are $5.607\times10^{-5}$ m/s
and $3.582\times10^{-4}$ rad/s. These checks use the 24-coordinate mission representation (augmented rank 36) rather than the 21-coordinate compiled tree (augmented rank 30) of Section~\ref{subsec:exp_mechanics}.

An independently evaluated Newton--Euler body sum supplies the
reduced inertia and bias, including curvature. Since independent
coordinates are actuator lengths in this experiment, its
sampled force-balance residual is
\begin{equation}
\begin{split}
r_{f,k}={}&M_{r,k}
 \frac{\dot\ell_{k+1}-\dot\ell_k}{\Delta t}
 +h_{r,k}-f_k\\
&-E_{T,k}^{\top}J_{w,k}^{\top}w_k,
\end{split}
\label{eq:exp_native_balance}
\end{equation}
where $J_w$ maps tree velocities to platform twist and $w_k$
is the applied platform wrench. The forward velocity difference
is paired with the force held over that step. The largest
per-actuator RMS residual across the four sampled missions is
2.612 mN. Unlike reference-trajectory comparison alone, this
check connects the compiled inertial data, measured motion,
curvature, and effort projection to the native plant response.

Numerical refinement provides a further check on integrated
motion and mechanical work. For Stewart, consecutive 4/2-ms
and 2/1-ms direct PACDM--Pinocchio runs reduce the maximum
paired position difference from 32.541 to $16.337\,\mu$m.
For Kangaroo, the maximum energy-ledger residual decreases
from 0.04108 to 0.02050 J when the MuJoCo step is halved from
25 to $12.5\,\mu$s, and from 0.20313 to 0.10207 J when
the Pinocchio step is halved from 1 to 0.5 ms. Halving the
PhysX step changes terminal pelvis position by 0.873 mm and
total actuator work by 0.08361 J; doubling position iterations
changes them by 0.03687 mm and $7.35\times10^{-4}$ J. Stewart
refinement appears in Fig.~\ref{fig:exp_curvature}(a).

\section{Conclusion}
\label{sec:conclusion}
RoboCompiler compiles a canonical mechanism graph of physical bodies,
joints, attachment frames, inertias, and actuator ports into one
interface for closure, motion, actuation, and dynamics. From a shared
assembled configuration, it generates the tangent lift, the
constraint-curvature correction, and virtual-work-dual actuator-port
maps, so that controllers, reduced dynamics, and external simulators
refer to the same assembly branch and actuator coordinates. On three
closed-chain mechanisms, RoboCompiler's constrained accelerations agree
with a 50-digit hard-constraint reference to a maximum normalized
discrepancy of $5.1\times10^{-7}$; on Stewart and Go2 at four times the
prescribed speed, the curvature correction lowers acceleration
residuals by about ten orders of magnitude. Generated cycle-local
evaluation reduces Kangaroo residual-and-Jacobian evaluation time by
96.7\% and the wall time of a complete closed-loop rollout by 66.8\%
while preserving the computed response, and models exported from the
same mechanism graph execute in MuJoCo, Pinocchio, and Isaac
Sim/PhysX, including native contact, material transfer, and
closed-chain landing. Future work will transfer the compiled
interfaces to hardware, extend the reduction through mechanism
singularities, and synthesize transition and reset laws for
connection changes during motion.

\bibliographystyle{IEEEtran}
\bibliography{lcsys}

@article{pakkila2017modeling,
  title={MODELING AND SIMULATION OF A SIX DEGREES OF FREE-DOM EXCAVATOR},
  author={P{\"a}kkil{\"a}, Sami},
  journal={Automation Technology, Engineering and Sciences, Tampere University of Technology, Tampere, Finland},
  year={2017}
}

@misc{kangaroo_sim2sim,
    author       = {{HuCeBot team, Inria/Loria}},
    title        = {{mujoco\_kangaroo\_sim2sim}},
    howpublished = {\url{https://github.com/hucebot/mujoco_kangaroo_sim2sim}},
    note         = {Commit c020b68, BSD 2-Clause license},
    year         = {2025}
  }

@misc{menagerie,
    author       = {Zakka, Kevin and Tassa, Yuval and {MuJoCo Menagerie Contributors}},
    title        = {{MuJoCo Menagerie}: A collection of high-quality simulation models for {MuJoCo}},
    howpublished = {\url{https://github.com/google-deepmind/mujoco_menagerie}},
    note         = {Commits 822c2d8 (\texttt{franka\_emika\_panda}, Apache-2.0 license) and 367e3d9 (\texttt{unitree\_go2}, BSD 3-Clause license)},
    year         = {2022}
  }

@inproceedings{shahna2025anti,
  title={Anti-slip AI-driven model-free control with global exponential stability in skid-steering robots},
  author={Shahna, Mehdi Heydari and Mustalahti, Pauli and Mattila, Jouni},
  booktitle={2025 IEEE/RSJ International Conference on Intelligent Robots and Systems (IROS)},
  pages={6855--6862},
  year={2025},
  organization={IEEE}
}

@inproceedings{shahna2024exponential,
  title={Exponential auto-tuning fault-tolerant control of n degrees-of-freedom manipulators subject to torque constraints},
  author={Shahna, Mehdi Heydari and Mattila, Jouni},
  booktitle={2024 IEEE 63rd Conference on Decision and Control (CDC)},
  pages={7263--7270},
  year={2024},
  organization={IEEE}
}

@misc{extended,
  author = {Batto, V. and De Matte{\"i}s, L. and Mansard, N.},
  title  = {Extended {URDF}: Accounting for Parallel Mechanism in Robot Description},
  year   = {2025},
  note   = {arXiv:2504.04767v2}
}

@article{morphology,
  author  = {D{\'i}az Ledezma, F. and Haddadin, S.},
  title   = {Machine Learning-Driven Self-Discovery of the Robot Body Morphology},
  journal = {Science Robotics},
  year    = {2023},
  volume  = {8},
  number  = {85},
  pages   = {eadh0972},
  doi     = {10.1126/scirobotics.adh0972}
}

@article{volpi,
  author  = {Volpi, D. J. and Wensing, P. M.},
  title   = {Adapting Rigid-Body Dynamics Derivatives for Constraint Embedding Closed-Chain Models},
  journal = {IEEE Robotics and Automation Letters},
  year    = {2026},
  volume  = {11},
  number  = {9},
  pages   = {10784--10791},
  doi     = {10.1109/LRA.2026.3709575}
}

@article{hqp,
  author  = {Escande, A. and Mansard, N. and Wieber, P.-B.},
  title   = {Hierarchical Quadratic Programming: Fast Online Humanoid-Robot Motion Generation},
  journal = {The International Journal of Robotics Research},
  year    = {2014},
  volume  = {33},
  number  = {7},
  pages   = {1006--1028}
}

@article{redundancy,
  author  = {Albu-Sch{\"a}ffer, A. and Sachtler, A.},
  title   = {Redundancy Resolution at Position Level},
  journal = {IEEE Transactions on Robotics},
  year    = {2023},
  volume  = {39},
  number  = {6},
  pages   = {4240--4261}
}

@misc{pbds,
  author = {Bylard, A. and Bonalli, R. and Pavone, M.},
  title  = {Composable Geometric Motion Policies Using Multi-Task Pullback Bundle Dynamical Systems},
  year   = {2021},
  note   = {arXiv:2101.01297v2}
}

@misc{parallel,
  author = {De Matte{\"i}s, L. and Batto, V. and Carpentier, J. and Mansard, N.},
  title  = {Optimal Control of Walkers with Parallel Actuation},
  year   = {2025},
  note   = {arXiv:2504.00642v1}
}

@book{featherstone,
  author    = {Featherstone, R.},
  title     = {Rigid Body Dynamics Algorithms},
  publisher = {Springer},
  year      = {2008}
}

@inproceedings{pinocchio,
  author    = {Carpentier, J. and others},
  title     = {The {Pinocchio} {C++} Library: A Fast and Flexible Implementation of Rigid Body Dynamics Algorithms and Their Analytical Derivatives},
  booktitle = {Proc. IEEE/SICE SII},
  year      = {2019},
  pages     = {614--619}
}

@article{bordalba2021kinodynamic,
  author  = {Bordalba, R. and Ros, L. and Porta, J. M.},
  title   = {A Randomized Kinodynamic Planner for Closed-Chain Robotic Systems},
  journal = {IEEE Transactions on Robotics},
  volume  = {37},
  number  = {1},
  pages   = {99--115},
  year    = {2021},
  doi     = {10.1109/TRO.2020.3010628}
}

@article{kumar2020hyrodyn,
  author  = {Kumar, S. and von Szadkowski, K. A. and Mueller, A. and Kirchner, F.},
  title   = {An Analytical and Modular Software Workbench for Solving Kinematics and Dynamics of Series-Parallel Hybrid Robots},
  journal = {Journal of Mechanisms and Robotics},
  volume  = {12},
  number  = {2},
  note    = {Art. no. 021114},
  year    = {2020},
  doi     = {10.1115/1.4045941}
}

@article{mingohoffman2024kangaroo,
  author  = {Mingo Hoffman, E. and Curti, A. and Miguel, N. and Kothakota, S. K. and Molina, A. and Roig, A. and Marchionni, L.},
  title   = {Modeling and Numerical Analysis of {Kangaroo} Lower Body Based on Constrained Dynamics of Hybrid Serial--Parallel Floating-Base Systems},
  journal = {Robotics and Autonomous Systems},
  volume  = {182},
  note    = {Art. no. 104827},
  year    = {2024},
  doi     = {10.1016/j.robot.2024.104827}
}

@inproceedings{chignoli2024urdfplus,
  author    = {Chignoli, M. and Slotine, J.-J. and Wensing, P. M. and Kim, S.},
  title     = {{URDF+}: An Enhanced {URDF} for Robots with Kinematic Loops},
  booktitle = {Proc. IEEE-RAS Int. Conf. Humanoid Robots (Humanoids)},
  address   = {Nancy, France},
  pages     = {197--204},
  year      = {2024},
  doi       = {10.1109/Humanoids58906.2024.10769903}
}

@article{chignoli2025propagation,
  author  = {Chignoli, M. and Adrian, N. and Kim, S. and Wensing, P. M.},
  title   = {A Propagation Perspective on Recursive Forward Dynamics for Systems With Kinematic Loops},
  journal = {IEEE Transactions on Robotics},
  volume  = {41},
  pages   = {5584--5603},
  year    = {2025},
  doi     = {10.1109/TRO.2025.3593081}
}

@misc{grbda,
  author       = {Chignoli, M. and Adrian, N. and Wensing, P. M.},
  title        = {{GRBDA}: Generalized Rigid-Body Dynamics Algorithms},
  howpublished = {\url{https://github.com/ROAM-Lab-ND/generalized_rbda}},
  note         = {Commit e68a6ec (2026-06-18)}
}

@article{sathya2026lcaba,
  author  = {Sathya, A. S. and Carpentier, J.},
  title   = {Constrained Articulated Body Algorithms for Closed-Loop Mechanisms},
  journal = {IEEE Transactions on Robotics},
  volume  = {42},
  pages   = {819--838},
  year    = {2026},
  doi     = {10.1109/TRO.2026.3651683}
}

@inproceedings{carpentier2019pinocchio,
  author    = {Carpentier, J. and Saurel, G. and Buondonno, G. and Mirabel, J. and Lamiraux, F. and Stasse, O. and Mansard, N.},
  title     = {The {Pinocchio} {C++} library: A fast and flexible implementation of rigid body dynamics algorithms and their analytical derivatives},
  booktitle = {Proc. IEEE/SICE Int. Symp. System Integration (SII)},
  address   = {Paris, France},
  pages     = {614--619},
  year      = {2019},
  doi       = {10.1109/SII.2019.8700380}
}

@misc{pinocchio410,
  title        = {Pinocchio v4.1.0},
  howpublished = {\url{https://github.com/stack-of-tasks/pinocchio/releases/tag/v4.1.0}},
  note         = {Released 2026-07-07 (tag commit 2ae7766)},
  year         = {2026}
}

@inproceedings{todorov2012mujoco,
  author    = {Todorov, E. and Erez, T. and Tassa, Y.},
  title     = {{MuJoCo}: A physics engine for model-based control},
  booktitle = {Proc. IEEE/RSJ Int. Conf. Intelligent Robots and Systems (IROS)},
  pages     = {5026--5033},
  year      = {2012},
  doi       = {10.1109/IROS.2012.6386109}
}

@misc{mujocochangelog,
  title        = {{MuJoCo} Documentation: Changelog, {Version 3.7.0 (April 14, 2026)} and {Version 3.14.0 (September 22, 2026)}},
  howpublished = {\url{https://mujoco.readthedocs.io/en/stable/changelog.html}},
  note         = {Accessed 2026-09-28}
}

@inproceedings{shahna2025model,
  title={Model-free generic robust control for servo-driven actuation mechanisms with layered insight into energy conversions},
  author={Shahna, Mehdi Heydari and Mattila, Jouni},
  booktitle={2025 American Control Conference (ACC)},
  pages={4331--4338},
  year={2025},
  organization={IEEE}
}

@misc{dastranj2026pacdm,
  author       = {Dastranj, M. and Mattila, J.},
  title        = {Modular Kinematic Reduction of Closed-Chain Mechanisms Using Path Assembly and Defect Homotopy},
  howpublished = {arXiv:2609.11338},
  year         = {2026}
}

@inproceedings{kumar2019hyrodynidetc,
  author    = {Kumar, S. and Mueller, A.},
  title     = {An Analytical and Modular Software Workbench for Solving Kinematics and Dynamics of Series-Parallel Hybrid Robots},
  booktitle = {Proc. ASME Int. Design Eng. Tech. Conf. Comput. Inf. Eng. Conf. (IDETC/CIE), 43rd Mechanisms and Robotics Conf.},
  volume    = {5A},
  address   = {Anaheim, CA, USA},
  note      = {Art. no. V05AT07A054},
  year      = {2019},
  doi       = {10.1115/DETC2019-97115}
}

\end{document}